\pdfoutput=1
\PassOptionsToPackage{table,svgnames}{xcolor}
\documentclass[11pt, a4paper, logo]{gdm/googledeepmind}

\makeatletter
\renewcommand\bibentry[1]{\nocite{#1}{\frenchspacing\@nameuse{BR@r@#1\@extra@b@citeb}}}
\makeatother
\pdftrailerid{redacted}
\renewcommand{\today}{}

\usepackage{gdm/gdm-colors}
\usepackage[authoryear,sort&compress,round]{natbib}
\usepackage{comment,dsfont,nicefrac,bbm,xspace}
\usepackage{algorithm,algpseudocode}
\usepackage{mathtools,multirow,makecell,appendix,nccmath,wrapfig}
\usepackage[most]{tcolorbox}
\usepackage{tikz,marvosym,bm,subcaption}
\usepackage[capitalize,noabbrev]{cleveref}

\usepackage{amsmath,amsfonts,bm}

\def\eqref#1{equation~\ref{#1}}

\def\1{\bm{1}}

\DeclareMathAlphabet{\mathsfit}{\encodingdefault}{\sfdefault}{m}{sl}
\SetMathAlphabet{\mathsfit}{bold}{\encodingdefault}{\sfdefault}{bx}{n}

\hypersetup{citecolor=[HTML]{2420cf},linkcolor=[HTML]{2420cf}}
\definecolor{darkgreen}{rgb}{0.0,0.5,0.0}

\definecolor{auditInk}{HTML}{263445}
\definecolor{auditRule}{HTML}{718196}
\definecolor{auditTraceFill}{HTML}{F1F3F5}
\definecolor{auditMonitorFill}{HTML}{E6EBF2}
\definecolor{auditDebugFill}{HTML}{E3F3F3}
\definecolor{auditCorrectionFill}{HTML}{F6F0E0}
\definecolor{auditSufficiencyFill}{HTML}{EFEBF5}
\definecolor{auditTraceText}{HTML}{52606D}
\definecolor{auditMonitorText}{HTML}{245D93}
\definecolor{auditDebugText}{HTML}{008C89}
\definecolor{auditCorrectionText}{HTML}{B58100}
\definecolor{auditSufficiencyText}{HTML}{7651AE}

\newcommand{\highlightbox}[1]{%
  {\setlength{\fboxsep}{1pt}%
   \colorbox{green!66!black!20}{#1}}%
}

\algrenewcommand\algorithmicrequire{\textbf{Input:}}
\definecolor{algcomment}{gray}{0.45}
\algrenewcommand{\algorithmiccomment}[1]{%
    \hfill{\small\color{captiongray}\(\triangleright\)~#1}}

\makeatletter
\g@addto@macro\fs@ruled{%
  \def\@fs@post{\kern2pt\hrule height0.8pt depth0pt\relax}%
}
\makeatother

\usetikzlibrary{positioning,calc}

\definecolor{textblue}{HTML}{2454C6}
\definecolor{textfill}{HTML}{EEF2FC}
\definecolor{textborder}{HTML}{B7C9F3}
\definecolor{latentgreen}{HTML}{008F88}
\definecolor{latentfill}{HTML}{E7F5F3}
\definecolor{latentborder}{HTML}{A9DCD7}
\definecolor{toolorange}{HTML}{B56A00}
\definecolor{toolfill}{HTML}{FFF6E6}
\definecolor{toolborder}{HTML}{E8C98E}
\definecolor{captiongray}{HTML}{53657A}
\definecolor{rulegray}{HTML}{CACDC5}

\renewcommand{\arraystretch}{1.12}
\renewcommand{\eqref}[1]{Eq.~\ref{#1}}

\theoremstyle{plain}
\newtheorem{theorem}{Theorem}[section]
\newtheorem{proposition}[theorem]{Proposition}

\theoremstyle{definition}

\theoremstyle{remark}
\newtheorem{remark}[theorem]{Remark}

\definecolor{stepblue}{HTML}{3B73A8}
\definecolor{steporange}{HTML}{C97531}
\definecolor{steppurple}{HTML}{8268A0}
\definecolor{stepteal}{HTML}{2F8178}

\newcommand{\stepcircle}[2][stepblue]{%
  \tikz[baseline=(n.base)]{
    \node[circle, fill=#1, text=white,
          inner sep=0pt, minimum size=1.1em,
          font=\small\sffamily\bfseries] (n) {#2};
  }%
}

\newcommand{\gaincell}[1]{\cellcolor{green!66!black!18}{\textbf{#1}}}

\definecolor{impr}{RGB}{34, 139, 34}
\definecolor{lightred}{RGB}{255, 230, 230}
\definecolor{darkred}{RGB}{192, 0, 0}

\definecolor{bestbg}{HTML}{FFF2B2}   %
\definecolor{secondbg}{HTML}{E8F0FE} %
\definecolor{LavenderLight}{HTML}{C7C3F5}
\definecolor{LightCoral}{RGB}{240,128,128}
\definecolor{LightBlue}{RGB}{173,216,230}

\definecolor{textfill}{HTML}{DCE8F0}
\definecolor{textborder}{HTML}{7893A6}
\definecolor{textblue}{HTML}{45697E}

\definecolor{latentfill}{HTML}{F3DDD2}
\definecolor{latentborder}{HTML}{C78F77}
\definecolor{latenttext}{HTML}{9D624D}

\definecolor{toolfill}{HTML}{DCECE7}
\definecolor{toolborder}{HTML}{739A8E}
\definecolor{tooltext}{HTML}{547D72}

\definecolor{captiongray}{HTML}{666666}
\definecolor{rulegray}{HTML}{8A8A8A}

\renewtcolorbox[auto counter, number within=section]{definition}[1]{
  enhanced,
  colback=orange!10!white,
  colframe=black!30,
  boxrule=0.6pt,
  top=7.5mm,
  bottom=.5mm,
  left=0.8mm,
  right=0.8mm,
  before skip=5pt, 
  after skip=10pt,
  fonttitle=\bfseries,
  coltitle=white,
  title={Definition~\thetcbcounter: #1},
  attach boxed title to top left={
    yshift=-7.8mm,
    xshift=1mm
  },
  boxed title style={
    colback=orange!70!black,
    rounded corners,
    boxrule=1pt,
    top=0.3pt, bottom=0.3pt, left=2pt, right=2pt
  }
}

\newtcolorbox{highlight}[1]{
  enhanced,
  colback=LightCoral!10,
  colframe=black!30,
  boxrule=0.6pt,
  top=7.5mm,
  bottom=.5mm,
  left=0.8mm,
  right=0.8mm,
  before skip=5pt,
  after skip=10pt,
  fonttitle=\bfseries,
  coltitle=white,
  title={#1},
  attach boxed title to top left={
    yshift=-7.8mm,
    xshift=1mm
  },
  boxed title style={
    colback=FireBrick,
    rounded corners,
    boxrule=1pt,
    top=0.3pt,
    bottom=0.3pt,
    left=2pt,
    right=2pt
  }
}

\tcolorboxenvironment{theorem}{
    colback=LightBlue!20,       %
    colframe=black,     %
    arc=1mm,               %
    top=.6mm,
    bottom=.6mm,
    right=1mm,
    boxrule=.5pt,           %
    fonttitle=\bfseries,   %
    left=1mm,              %
    attach boxed title to top left={yshift=-2mm, xshift=4mm},
    boxed title style={
        colback=black!40,
        arc=1mm,
    },
    before skip=6pt,   %
    after skip=6pt,    %
}

\tcolorboxenvironment{proposition}{
    colback=LightBlue!20,       %
    colframe=black,     %
    arc=1mm,               %
    top=.6mm,
    bottom=.6mm,
    right=1mm,
    boxrule=.5pt,           %
    fonttitle=\bfseries,   %
    left=1mm,              %
    attach boxed title to top left={yshift=-2mm, xshift=4mm},
    boxed title style={
        colback=black!40,
        arc=1mm,
    },
    before skip=6pt,   %
    after skip=6pt,    %
}

\tcolorboxenvironment{lemma}{
    colback=red!10,       %
    colframe=black!,     %
    arc=1mm,               %
    boxrule=1pt,           %
    fonttitle=\bfseries,   %
    left=1mm,              %
    attach boxed title to top left={yshift=-2mm, xshift=4mm},
    boxed title style={
        colback=black!40,
        arc=1mm,
    }
}

\tcolorboxenvironment{remark}{
    colback=LightCoral!10,       %
    colframe=black,     %
    arc=1mm,               %
    top=.6mm,
    bottom=.6mm,
    boxrule=.5pt,           %
    right=1mm,
    fonttitle=\bfseries,   %
    left=1mm,              %
    attach boxed title to top left={yshift=-2mm, xshift=4mm},
    boxed title style={
        colback=black!40,
        arc=1mm,
    },
    before skip=6pt,   %
    after skip=6pt,    %
}

\newcommand{\ours}{HybridAgent\xspace}
\newcommand{\link}{HybridLink\xspace}

\definecolor{hySlate}{HTML}{3A506B}\definecolor{hyPaper}{HTML}{F7F8FA}
\definecolor{hyTextAcc}{HTML}{3A6EA5}\definecolor{hyLatAcc}{HTML}{D2641C}\definecolor{hyToolAcc}{HTML}{2E8B57}\definecolor{hyObsAcc}{HTML}{6E6E6E}

\title{Hybrid Reasoning Agents}
\author[1,2*$\dagger$]{Jiaru Zou}
\author[3,4$\dagger$]{Yunzhe Qi}
\author[3$\dagger$]{Sirui Chen}
\author[2]{Federico Bianchi}
\author[2]{Shang Zhu}
\author[1]{Pan Lu}
\author[3]{Jingrui He}
\author[1,2]{James Zou}

\affil[1]{Stanford University }
\affil[2]{Together AI }
\affil[3]{UIUC }
\affil[4]{Meta\newline}
\affil[*]{Project Lead }
\affil[$\dagger$]{Core Contributor}

\renewcommand{\abstractprojectline}{{%
    \normalfont\fontsize{10}{12}\selectfont\color{black}%
    \hypersetup{urlcolor=[HTML]{2420cf}}%
    \makebox[1em][c]{\raisebox{-0.125em}{\includegraphics[height=0.95em]{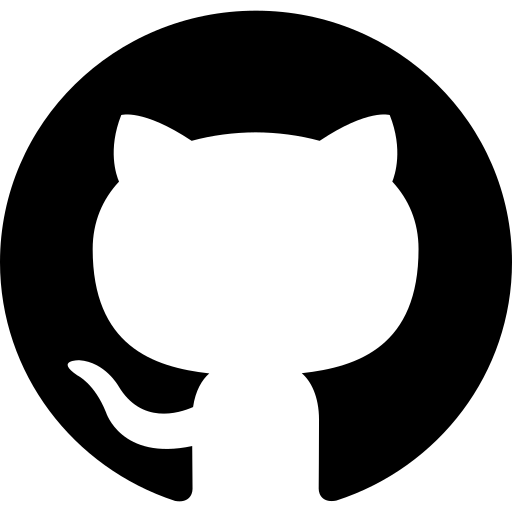}}}\hspace{0.25em}\textbf{Code:} \href{https://github.com/togethercomputer/HybridAgent}{https://github.com/togethercomputer/HybridAgent}\par
    \vspace{2pt}\noindent\makebox[1em][c]{\raisebox{0.05em}{\includegraphics[height=0.6em]{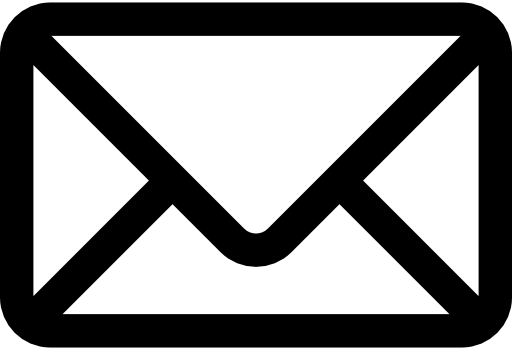}}}\hspace{0.25em}\textbf{Correspondence:} \href{mailto:jiaru@stanford.edu}{\textcolor[HTML]{2420cf}{jiaru@stanford.edu}}, \href{mailto:jamesz@stanford.edu}{\textcolor[HTML]{2420cf}{jamesz@stanford.edu}}%
}}
\begin{abstract}
How can we build an agent that thinks beyond human language, yet still acts through it in the external environment? Modern LLM agents still reason and interact through text, while latent reasoning offers an efficient internal medium but lacks native tool access and reasoning transparency. 
We bridge these two directions with \textbf{Hybrid Reasoning Agents (HybridAgents)}, which dynamically navigate between latent reasoning and text-based interaction throughout a unified generation process.
Architecturally, a lightweight HybridLink enables seamless transitions between text and latent modes, while the agent's continuous working memory preserves lossless hybrid context across multiple turns.
These designs preserve tool-integrated agentic capabilities while making latent thoughts directly inspectable.
Building on this architecture, we further develop dedicated post-training algorithms for hybrid reasoning: \textbf{HybridSFT} stabilizes tool-integrated hybrid generation by teaching the backbone to coordinate with HybridLink, then \textbf{HybridRL} scales reasoning by learning when to think in latent space and when to act through text.
We evaluate HybridAgent at 3 model scales across 11 benchmarks covering mathematics, science, medicine, search, and agentic reasoning. Compared with frontier LLM Agents and agentic frameworks spanning both text-based and latent methods, HybridAgent consistently improves performance by 10.4\%, while delivering 35.5\%--54.7\% token reduction, 1.8$\times$--2.6$\times$ faster inference, and 2.7$\times$--3.7$\times$ cost savings.
\end{abstract}

\begin{document}
\maketitle
\addtocontents{toc}{\protect\setcounter{tocdepth}{-1}}

\begin{figure}[!h]
    \centering
    \vspace{3pt}
    \includegraphics[width=\linewidth]{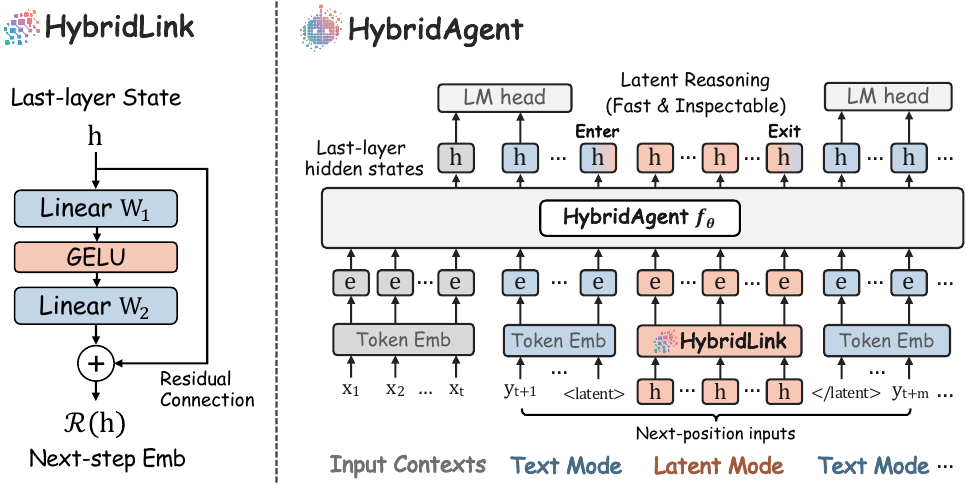}
    \caption{
    \textbf{\ours Architecture Overview.}
    \textcolor{latenttext}{\textbf{\textit{Left:}}} A \link maps last-layer hidden states back into the input embedding space, letting the agent enter latent mode for \textit{recurrent, fast, inspectable inner thinking} and return to text mode for \textit{tool calls and external monitoring}.
    \textcolor{textblue}{\textbf{\textit{Right:}}} \ours performs hybrid reasoning that thinks in latent space and acts through text within a single generation process.
    Building on this architecture, \ours \textcolor{latenttext}{\textbf{(i)}} achieves \textbf{multi-turn hybrid generation} by carrying the entire hybrid context through its own working memory (Section~\ref{sec:architecture}), and \textcolor{textblue}{\textbf{(ii)}} can be converted from any standard pre-trained LLM through a dedicated \textbf{post-training pipeline} (Section~\ref{sec:training}).
    }
    \label{fig:hybrid_agent}
    \vspace{-15pt}
\end{figure}

\newpage
\vspace{-2pt}
\section{Introduction}
\vspace{-2pt}

Frontier intelligence no longer scales through the model's reasoning alone, but depends increasingly on its ability to interact with the external world \citep{kimik3, kwa2026measuring, sun2026agents}.
To solve a hard problem, an LLM agent must reach beyond the model itself by continually invoking specialized tools, observing feedback from the environment, and adapting its next move \citep{tool_survey, silver2025welcome, feng2026retool}.
Beneath this expanding repertoire of agent harnesses lies a single lingua franca: \textit{text}, the common medium of both the agent's autoregressive generation and its environment interaction \citep{sumers2024cognitive, zhang2026darwin}. 
All intermediate reasoning trajectories are generated token by token as text before the agent can act on the environment.

A separate line of work on latent reasoning points in a different direction \citep{zhu2025survey, chen2025reasoning}.
By unfolding computation directly in continuous hidden space, latent models bypass the need to decode and verbalize every intermediate step, providing more efficient internal reasoning within the model's ``own language" \citep{hao2024training,wang2025monet,yu2026latent}.
However, today’s latent reasoners remain largely sealed behind the model boundary.
Unlike decoded text, continuous latent vectors (i) lack native pathways for tool invocation across multi-turn agentic reasoning \citep{chang2026unlocking}, and (ii) do not readily expose human-interpretable semantics, leaving them difficult to supervise during training \citep{chen2026what} and to audit or debug at inference \citep{korbak2025chain}.
These constraints keep latent reasoning models from operating as capable agents in their own right.

Taken together, agentic interaction and efficient inner reasoning have so far required a different medium each, leaving the two lines of work apart in the development of next-generation agents.
We believe this division can be closed by building a new type of agent that combines both strengths to \textit{think beyond text while still acting through it}. We call this unified design \textbf{\ours}, one that dynamically navigates between latent reasoning and text-based interaction within a single generation.

Building such an agent is hard, since it needs to handle both dense continuous vectors and sparse discrete tokens as hybrid inputs, and determine when to transition between text and latent modes to balance agentic actions and reasoning efficiency.
To offset the hardness, we first revisit the model architecture and extend its autoregressive generation interface to support hybrid inputs.
At the heart of this design is \textbf{\link}, a lightweight residual adapter that bridges latent thoughts from the last-layer hidden space back into the input embedding space and enables sequence-wise recurrence within latent mode.
Then, we realize multi-turn hybrid generation through the agent's own working memory, carrying complete hybrid trajectories of text, latent thoughts, and tool observations across turns.
Surprisingly, we find that \ours can render its thoughts in latent mode as human-readable pseudo-text, making latent reasoning transparent and directly inspectable rather than opaque.

To further unlock the potential of hybrid reasoning, we equip \ours with dedicated post-training algorithms.
\textbf{HybridSFT} first stabilizes tool-integrated hybrid generation through activation and self-conditioned training, which adapts the backbone to reason over latent thoughts produced by its own \link, enabling reliable coordination between the two components and closing the train-inference mismatch.
Then, \textbf{HybridRL} scales agentic reasoning capabilities by training \ours to adaptively balance latent reasoning and text-based interaction. To realize this, we enable policy optimization over hybrid trajectories through hybrid policy factorization and exact rollout replay.

We provide theoretical grounding for HybridAgent through complementary analyses of multi-turn context continuation and hybrid policy optimization. From an architectural standpoint, we show that continuous working memory losslessly preserves the entire hybrid context across different turns. From a learning perspective, the hybrid policy factorization and exact latent-context replay in HybridRL provide a well-founded basis for policy updates over complete hybrid trajectories.

Empirically, we evaluate \ours on 11 benchmarks spanning mathematics, science, medicine, knowledge-intensive search, and challenging agentic reasoning.
We instantiate and train \ours with Qwen3-Base models \citep{yang2025qwen3} across 4B, 8B, and 14B scales.
Compared with frontier LLM agents and agentic frameworks spanning both text-based and latent reasoning, \ours achieves an average performance improvement of 10.4\% and exhibits strong test-time scalability. In terms of efficiency, \ours reduces token usage by 35.5\%--54.7\%, delivers 1.8$\times$--2.6$\times$ inference speedup, and lowers computational cost by 2.7$\times$--3.7$\times$. 
Additional in-depth analyses of training dynamics, reasoning auditability, and agentic tool-use reliability reveal a clear division of labor between the two reasoning modes: \ours shifts more internal reasoning into semantically rich and interpretable latent thoughts, while reserving text for external inspection and tool interaction.

\section{Preliminary}
\label{sec:pre}

\textbf{Hybrid Generation in LLMs.} We consider an LLM agent $f_\theta$ built upon a standard Transformer parameterized by $\theta$. To enable hybrid generation, the agent operates autoregressively in two modes: (i) \textcolor{textblue}{\texttt{text}}, which performs explicit generation in the token space, and (ii) \textcolor[HTML]{C87945}{\texttt{latent}}, which performs implicit generation in the continuous representation space. Formally, given an input context $x$ with corresponding input embeddings $E=[e_1,\dots,e_t] \in \mathbb{R}^{t \times d_h}$, the agent computes the last-layer hidden state $h_t$ through a forward pass, and the hybrid input representation for the next step $t+1$ is:
\begin{equation}
\begin{aligned}
e_{t+1} =
\begin{cases}
    \operatorname{Emb}(y_{t+1}), 
    \quad y_{t+1} \sim \operatorname{Softmax}(W_{\mathrm{out}}h_t),
    & \text{if } \textcolor{textblue}{\texttt{text}} \text{ mode}, \\[2pt]
    h_t,
    & \text{if } \textcolor[HTML]{C87945}{\texttt{latent}} \text{ mode},
\end{cases}
\end{aligned}
\end{equation}
where $W_{\mathrm{out}}$ is the vocabulary projection matrix, $\operatorname{Emb}(\cdot)$ is the token embedding function, and $h_t$ is directly reused as the next latent input in standard latent reasoning \citep{hao2024training,cheng2024compressed,coda2025exploring}.
The next hidden state of $f_\theta$ is then computed as:
\begin{equation}
\begin{aligned}
h_{t+1}=f_\theta([E_{\le t};e_{t+1}]).
\end{aligned}
\end{equation}
Throughout inference, the hybrid agent dynamically switches between the two generation modes, producing a hybrid trajectory that balances the model's inner reasoning with external interaction.

\textbf{Tool-Integrated Agentic Reasoning.} A hybrid agent can leverage its text mode to interact with external environments.
Similar to standard LLM agents, the hybrid agent generates a tool call as a sequence of text, invokes the corresponding tool (e.g., Python, Search), and receives the resulting observations. A complete single-turn \textit{hybrid reasoning trajectory} $\tau$ can be represented as:

\begingroup
\setlength{\abovedisplayskip}{6pt}
\setlength{\belowdisplayskip}{6pt}
\setlength{\abovedisplayshortskip}{6pt}
\setlength{\belowdisplayshortskip}{6pt}
\begin{equation}
\label{eq:hybrid_trajectory}
\vcenter{\hbox{%
\resizebox{\dimexpr\linewidth-2.6em\relax}{!}{%
\begin{tikzpicture}[
    font=\sffamily,
    box/.style={
        rounded corners=2pt,
        minimum height=0.65cm,
        inner xsep=9pt,
        draw,
        line width=0.65pt,
        font=\ttfamily
    },
    textbox/.style={
        box,
        fill=textfill,
        draw=textborder,
        text=textblue
    },
    latentbox/.style={
        box,
        fill=latentfill,
        draw=latentborder,
        text=latenttext
    },
    toolbox/.style={
        box,
        fill=toolfill,
        draw=toolborder,
        text=tooltext
    },
    caption/.style={
        text=captiongray,
        font=\sffamily\small
    }
]

\node[textbox] (text1)
    {\texttt{[text]}};

\node[latentbox, right=0.28cm of text1] (latent)
    {\(\langle\)\texttt{latent}\(\rangle\)};

\node[right=0.18cm of latent, text=black] (dots)
    {\(\cdots\)};

\node[textbox, right=0.18cm of dots] (text2)
    {\texttt{[text]}};

\node[toolbox, right=0.65cm of text2] (call)
    {\texttt{[tool\_call]}};

\node[toolbox, right=2.65cm of call] (response)
    {\texttt{[tool\_response]}};

\coordinate (arrowmid) at
    ($(call.east)!0.5!(response.west)$);

\draw[->, line width=0.8pt, draw=black]
    ([xshift=-0.38cm, yshift=0.16cm]arrowmid)
    --
    ([xshift=0.38cm, yshift=0.16cm]arrowmid);

\node[
    caption,
    text=black,
    anchor=north,
    yshift=0.04cm
] at (arrowmid)
    {Tool Execution};

\coordinate (modelL) at
    ([yshift=-0.25cm]text1.south west);
\coordinate (modelR) at
    ([yshift=-0.25cm]text2.south east);

\draw[rulegray, line width=1pt]
    (modelL) -- (modelR);

\node[
    caption,
    text=black,
    font=\sffamily\normalsize,
    anchor=north,
    yshift=-0.16cm
] at ($(modelL)!0.5!(modelR)$)
    {Model-Side Hybrid Reasoning};

\coordinate (envL) at
    ([yshift=-0.25cm]call.south west);
\coordinate (envR) at
    ([yshift=-0.25cm]response.south east);

\draw[rulegray, line width=1pt]
    (envL) -- (envR);

\node[
    caption,
    text=black,
    font=\sffamily\normalsize,
    anchor=north,
    yshift=-0.16cm
] at ($(envL)!0.5!(envR)$)
    {Environment-Side Interaction};

\end{tikzpicture}%
}
}}
\end{equation}
\endgroup
where $[\cdot]$ and $\langle \cdot \rangle$ denote text chunks and latent vectors, respectively.
The whole trace is then appended as input for the next generation turn, and the process continues until a final text answer is produced.

\section{Building a Hybrid Reasoning Agent}
\label{sec:architecture}

We introduce \ours{}, a hybrid framework that dynamically navigates between explicit text and continuous latent reasoning within a single unified generation process.
As illustrated in Figure \ref{fig:hybrid_agent}, the architecture preserves an LLM agent's autoregressive generation and tool-calling interfaces while accommodating both text and latent inputs.
Below, we first present the hybrid interface for seamless text-latent mode transition (Section \ref{sec:mode_tran}), followed by multi-turn hybrid generation (Section \ref{sec:multi_turn_generate}).

\subsection{Hybrid Mode Interface}
\label{sec:mode_tran}

Existing latent reasoning models typically feed the last-layer hidden state $h_t$ directly back as the next-step input $e_{t+1}$. 
This design is not well suited to hybrid generation, as substituting $h_t$ for a token embedding creates a representation mismatch with the pretrained input space, making transitions between continuous latent states and discrete tokens difficult.
To support both latent and text modes in a shared input interface, we add a lightweight adapter \textit{\link} to work alongside the backbone.

\textbf{\link as a Bridge.} An LLM agent's last-layer hidden states naturally encode the semantics of its current generation, so the \link $\mathcal{R}$ is designed to carry this information from the hidden representation space back into the input embedding space.
Formally, given any last-layer hidden state $h_t$, \link simulates its corresponding next-step input embedding as:
\begin{equation}
\label{eq:hybridlink}
\begin{aligned}
    e_{t+1} = \mathcal{R}(h_t) = h_t + W_2\,\sigma\!\left(W_1 h_t\right),
\end{aligned}
\end{equation}
where $W_1, W_2$ are standard linear projections, $\sigma$ is an activation function (e.g., GELU).
As the residual connection largely preserves the original semantics of $h_t$, \link only needs to align the distributional difference rather than learn the entire projection from scratch. 
Accordingly, the link module is lightweight, with approximately $2d_h^2$ parameters (e.g., $\sim$10M parameters for a 4B LLM).

\textbf{Dynamic Text-Latent Mode Transition.} 
Next, we describe how \ours transitions between text and latent modes during generation, together with the entry and exit decisions for each mode.

In text mode, \ours behaves exactly as a standard LLM, sampling the next token from the vocabulary at each step. The process continues until the agent generates one special token \texttt{<latent>}, which serves as the entry anchor for the latent mode. For the next forward pass, the agent uses \link for latent reasoning in place of the LM head and token embedding layer.

\begin{figure}[!t]
    \centering
    \includegraphics[width=\linewidth]{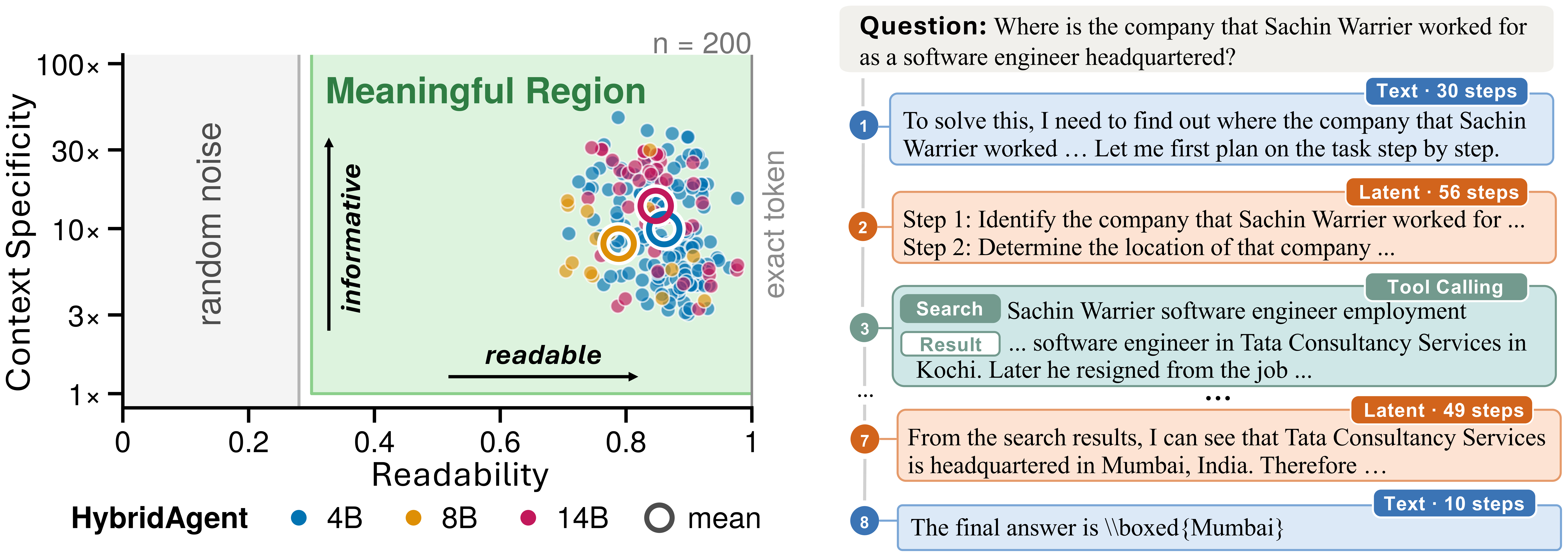}
    \caption{
    \textbf{Interpretable Latent Mode in \ours.} All latent steps can be expressed as readable pseudo-text carrying their semantics for external inspection.
    \textit{Left:} Quantitatively, latent thoughts are highly readable and context-specific (see Appendix \ref{app:latent_probe}).
    \textit{Right:} Qualitatively, latent pseudo-text stays semantically aligned with surrounding text reasoning and tool calls. (``Steps'' counts forward passes).
    }
    \label{fig:probe_method}
\end{figure}

In latent mode, each step's last-layer hidden state is passed through \link via \eqref{eq:hybridlink}, allowing \ours to continue its generation flow entirely in latent space without token decoding and re-embedding.
This sequence-wise recurrent computation over continuous representations extends the agent's effective reasoning depth, enabling more expressive internal reasoning.
The latent reasoning span terminates when the semantics of the last-layer hidden state $h_t$ align with the exit anchor \texttt{</latent>}. Specifically, at any latent step $t$, the latent mode closes by satisfying the condition:
\begin{equation}
\label{eq:latent-exit}
    \operatorname*{arg\,max}_{v \in \mathcal{V}} \; \cos\bigl(\mathcal{R}(h_t),\, \operatorname{Emb}(v)\bigr) = \texttt{</latent>}.
\end{equation}
\ours then returns to its text mode, and this transition continues during generation.
At each step, the agent adaptively decides whether to remain in its current mode or switch, allowing the length and frequency of text and latent spans to emerge dynamically rather than follow preset configurations.

\textbf{Interpretable Latent Mode.} During hybrid generation, a useful property of \ours is that its latent mode remains decoding-free yet directly interpretable. 
To inspect any latent reasoning span, one can easily recover its semantics by mapping each latent step produced by \link to its nearest token embedding, as in \eqref{eq:latent-exit}. We illustrate this interpretability property clearly in Figure \ref{fig:probe_method}.

\subsection{Multi-Turn Hybrid Generation over Working Memory}
\label{sec:multi_turn_generate}

In standard LLM agents, each turn's generated content and tool outputs are appended to the textual context and re-encoded as the next turn's input prefix.
Hybrid trajectories, however, interleave text tokens with latent vectors and can't be reconstructed through conventional tokenization alone. Instead, \ours carries the complete context across turns through its continuous working memory.

\begin{figure}[!t]
    \centering
    \includegraphics[width=\linewidth]{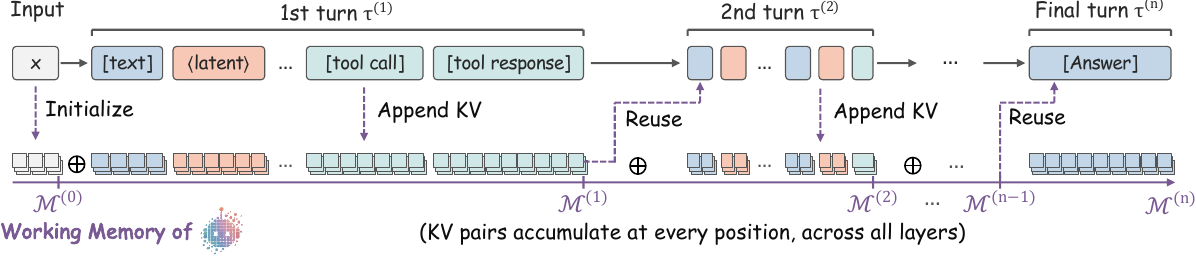}
    \caption{\textbf{Multi-Turn Hybrid Generation over \ours's Own Working Memory.} \ours accumulates key-value (KV) entries from text, latent thoughts, and tool observations in continuous working memory, directly reusing and updating it across turns to preserve the full hybrid context.}
    \label{fig:KV}
\end{figure}

\textbf{Cross-Turn Working Memory.}
Recall that a single-turn hybrid trajectory $\tau$ in \eqref{eq:hybrid_trajectory} consists of model-generated text, latent thoughts, and external tool interactions.
During input construction, all contents are mapped into the same input space through either the input embedding layer or \link. 
This shared space allows the Transformer to process all three components uniformly and collect their KV entries throughout generation.
We leverage these accumulated KV entries as the \ours's continuous working memory to achieve multi-turn hybrid generation.

As illustrated in Figure \ref{fig:KV}, \ours first initializes its working memory $\mathcal{M}$ from the input context $x$ and then accumulates the KV entries generated throughout the first-turn trajectory $\tau^{(1)}$, i.e.,
\begin{equation}
\label{eq:working_memory}
\begin{aligned}
\mathcal{M}^{(0)} &= \bigl[K_{\le |x|},\, V_{\le |x|}\bigr],\\
\mathcal{M}^{(1)} &= \mathcal{M}^{(0)} \oplus \bigl[K_{\le t},V_{\le t}\bigr],
\quad t = 1,\dots,|\tau^{(1)}|.
\end{aligned}
\end{equation}
Here, $\mathcal{M}^{(1)}$ preserves the complete first-turn hybrid context in positional order across all Transformer layers.
For each subsequent turn, the agent directly resumes generation from the current working memory and updates it by appending newly generated KV entries at each step.
This process continues across the sequence of hybrid trajectories $\{\tau^{(1)},\tau^{(2)},\dots,\tau^{(n)}\}$ until the final answer is produced.
We further provide a formal justification for this continuation process in Appendix~\ref {app:theorey_memory}.
Overall, multi-turn hybrid generation proceeds over one continuous working memory flow, preserving the complete hybrid context across turns without requiring intermediate token-space reconstruction.

\section{Learning to Perform Hybrid Reasoning}
\label{sec:training}

With the overall framework in place, we now turn to training. 
As \ours leaves the backbone architecture intact, no pre-training from scratch is required. Instead, we adapt the standard agentic post-training pipeline of supervised fine-tuning and reinforcement learning to the hybrid setting. The two training stages serve complementary roles: SFT teaches the agent to coordinate with \link and instantiate reliable tool use for stable hybrid generation (Section \ref{sec:HybridSFT}), while RL further scales reasoning by learning when to think in latent space and when to act through text (Section \ref{sec:hybridrl}).

\begin{figure}[!t]
\centering
\includegraphics[width=\linewidth]{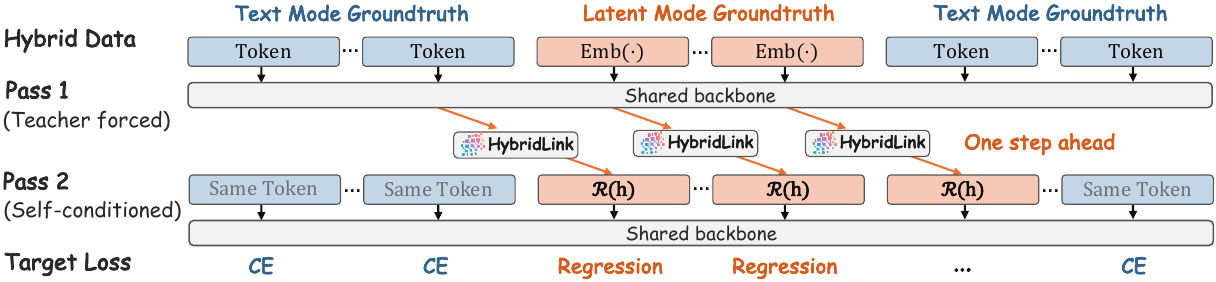}
\caption{
\textbf{Self-Conditioned HybridSFT.}
A teacher-forced first pass produces next-step latent inputs via \link. A second pass replaces ground-truth latent embeddings Emb$(\cdot)$ with the predictions $\mathcal{R}(h)$ and jointly trains the backbone and \link using CE at text and regression at latent positions.
}
\label{fig:hybrid_SFT}
\end{figure}

\subsection{Stabilize Generation via Hybrid Supervised Fine-tuning}
\label{sec:HybridSFT}

As a first step toward stabilizing hybrid reasoning in \ours, we design the entire HybridSFT stage through a hybrid data curation pipeline and a staged training procedure.

\textbf{Hybrid Data Recipe.} To curate high-quality training data, we filter existing text-based agentic SFT datasets from multiple sources and convert them into hybrid trajectories through automatic annotation.
We assign reasoning steps to latent mode based on both LLM Experts' judgments and target models' own prediction entropy, while retaining a small set of key steps for tool interactions and external debugging/auditing in text mode (see Appendix \ref{app:annotation} for data curation details).

\textbf{Format and \link Activation.}
As a warm start, we adapt the backbone to the hybrid data format and initialize the \link.
For format activation, we perform standard text SFT to teach the backbone when to enter and exit latent mode by emitting the two special anchor tokens.
We then freeze the backbone and train \link to align last-layer hidden states at latent positions with their target input embeddings. For a training instance $(x,y)\in \mathcal{D}_\text{SFT}$, the activation objectives are:
\begin{equation}
\label{eq:hybridsft_activation}
\begin{aligned}
&\mathcal{L}_{\text{Backbone}} = \mathrm{CE}\bigl(f_\theta(x),\, y\bigr),\\
&\mathcal{L}_{\text{\link}} = \lambda_{\cos}\Bigl(1 - \cos\bigl(\mathcal{R}(h_\text{lat}),\, \operatorname{Emb}(y_\text{lat})\bigr)\Bigr) + \lambda_{\mathrm{mse}}\bigl\lVert \mathcal{R}(h_\text{lat}) - \operatorname{Emb}(y_\text{lat}) \bigr\rVert_2^2,
\end{aligned}
\end{equation}
where $h_{\mathrm{lat}}$ is the last-layer hidden states at latent positions, $y_{\mathrm{lat}}$ is the corresponding ground-truth text tokens, and $\lambda_{\cos},\lambda_{\mathrm{mse}}$ balance directional and magnitude alignment, respectively.

\textbf{Self-Conditioned HybridSFT.}
The activation provides a well-aligned initialization, but the backbone has not yet learned to reason over latent thoughts produced by its own \link (see Appendix~\ref{app:self_condition} for details on this train-inference mismatch).
We therefore adopt a \textit{two-pass self-conditioned training} scheme as illustrated in Figure~\ref{fig:hybrid_SFT}.
In the first pass, \ours is teacher-forced on the ground-truth hybrid sequence. \link maps the generated last-layer hidden states at latent positions one step ahead to form latent inputs for the next pass.
In the second pass, \ours processes the same sequence but replaces the ground-truth latent embeddings with its own latent inputs. We jointly supervise text and latent positions on the second pass as:
\begin{equation}
\label{eq:hybridsft_twopass}
    \mathcal{L}_{\text{HybridSFT}}
    = \mathcal{L}_{\text{Backbone}}
    + \mathcal{L}_{\text{\link}},
\end{equation}
where the CE loss $\mathcal{L}_{\text{Backbone}}$ and regression loss $\mathcal{L}_{\text{\link}}$ are applied only at text and latent positions, respectively.
Overall, HybridSFT trains the agent to coordinate its backbone with \link for coherent reasoning across text and latent modes, supporting stable hybrid generation at inference.

\subsection{Scale Reasoning via Hybrid Reinforcement Learning}
\label{sec:hybridrl}

\begin{figure}[!t]
\centering
\begin{minipage}[t]{\textwidth}

\begin{algorithm}[H]
\caption{Hybrid Reinforcement Learning (HybridRL)}
\label{alg:hybrid_rl}

\begingroup
\algrenewcommand\algorithmicindent{0.8em}

\begin{algorithmic}[1]
\Require Agent $f_\theta$ (SFT initialized in Section \ref{sec:HybridSFT}), \link $\mathcal{R}$, task set $\mathcal{D}_\text{RL}$, group size $G$, $\eta$
\For{each training step}
    \State Sample $x \sim \mathcal{D}_{\mathrm{RL}}$;
Collect $\{\tau_i\}_{i=1}^{G} \leftarrow f_\theta(x)$
    \Comment{Hybrid Policy Rollout}

    \State Record $\tau_i$ with latent $z_i=\mathcal{R}(H_i)$ (\eqref{eq:hybridlink}), text $y_i$, and tool observations $o_i$  

    \State Compute rewards $r_i$ (\eqref{eq:rl_reward}) and group-normalized advantages $A_i$

    \State Replay $\tau_i$ with $(z_i,y_i,o_i)$ to obtain $\pi_\theta(y_{i,t}\mid\tau_{i,<t})$
    \Comment{Exact Rollout Replay}

    \State Compute $\rho_{i,t}$ via \eqref{eq:rl_ratio} and \eqref{eq:rl_factorization}
    \Comment{Hybrid Policy Factorization}

    \State Update $\theta \leftarrow \theta-\eta\nabla_\theta\mathcal{L}_{\mathrm{RL}}$ (standard RL objective)
    \Comment{Hybrid Policy Update}

\EndFor
\end{algorithmic}
\endgroup
\end{algorithm}
\end{minipage}
\end{figure}

A \ours cannot straightforwardly adopt popular RL algorithms such as PPO \citep{schulman2017proximal} or GRPO \citep{shao2024deepseekmath}, as they assume every generated position carries a token-level likelihood.
Specifically, at each position $t$, standard RL computes the per-token importance ratio as:
\begin{equation}
\label{eq:rl_ratio}
\begin{aligned}
\rho_t = \frac{\pi_\theta\bigl(y_t \mid y_{<t}\bigr)}{\pi_{\theta_{\mathrm{old}}}\bigl(y_t \mid y_{<t}\bigr)},
\end{aligned}
\end{equation}
where $\pi_\theta, \pi_{\theta_{\mathrm{old}}}$ denote the current and rollout policies.
At latent positions, however, \ours generates continuous vectors rather than sampled tokens, leaving no token-level likelihood. Thus, the importance ratio $\rho_t$ is undefined.
We address this incompatibility with two simple components.

\textbf{Hybrid Policy Factorization.}
Recall that, in a hybrid trajectory, latent thoughts are computed entirely in the continuous latent space through \link, and tool responses are the masked observations from the environment. 
Thus, policy sampling during rollout is confined to model-generated text positions.
To justify why policy optimization needs only to operate over these sampled positions, we formalize the \textit{hybrid policy factorization} below with detailed theoretical analyses.
\vspace{4pt}
\begin{proposition}[Hybrid Policy Factorization]
\label{prop:factorization}
Let $\tau$ be a hybrid trajectory generated by the policy $\pi_\theta$ (induced by $f_\theta$) with fixed \link $\mathcal{R}$. Then the hybrid policy factorizes as:
\begin{equation}
\label{eq:rl_factorization}
\begin{aligned}
    \pi_\theta(\tau \mid x) \;=\; \prod_{t \in \mathcal{T}} \pi_\theta\bigl(y_t \mid x,\, \tau_{<t}\bigr),  
\end{aligned}
\end{equation}
where $\mathcal{T}$ is the set of sampled text positions of trajectory $\tau$.
\end{proposition}
\vspace{4pt}
The proof and additional theoretical justifications are provided in Appendix \ref{app:theorey_rl}.
Proposition~\ref{prop:factorization} shows that likelihood-ratio updates over sampled text positions are sufficient to optimize the full hybrid trajectory, making the importance ratio in \eqref{eq:rl_ratio} well-defined for the \ours policy.

\textbf{Policy Update via Exact Rollout Replay.}
Policy optimization generally involves rebuilding a differentiable computation graph for the rollouts.
Applying this to \ours further requires recomputing sampled text-token likelihoods conditioned on the same preceding latent thoughts used during rollout.
We satisfy this requirement via \textit{exact latent-context replay}.
During hybrid rollout, we record latent thoughts alongside sampled text tokens and tool observations.
During policy updates, we replay the full hybrid trajectory in embedding space with a single differentiable forward pass, keeping the recorded latent inputs unchanged while recomputing the sampled likelihoods in \eqref{eq:rl_factorization}.

Algorithm \ref{alg:hybrid_rl} outlines the full HybridRL training procedure. We keep the \link fixed to preserve deterministic latent transitions.
For the reward design, we simply formulate it as:
\begin{equation}
\label{eq:rl_reward}
r(\tau) = r_{\mathrm{correct}} + r_{\mathrm{format}} + r_{\mathrm{tool}},
\end{equation}
where $r_{\mathrm{correct}}$ measures answer correctness, $r_{\mathrm{format}}$ penalizes invalid formatting, and $r_{\mathrm{tool}}$ encourages multi-tool use.
We impose no explicit reward for mode usage, allowing \ours to adaptively balance text and latent generation and discover effective hybrid reasoning strategies during RL.

\section{Empirical Evaluations}
\label{sec:exps}

\begin{table*}[!t]
\centering
\caption{
\textbf{Main Results of \ours over Agentic Reasoning Tasks.} We report GAIA and WebWalker by difficulty level, HLE by category (NS/CE/SF), SealQA on the Seal-0 (``S-0") and Seal-Hard (``Hard") splits, and DSQA by full-correct rate (``Acc.") and F1 score. Across 5 independent runs, the average standard deviations of \ours are $\pm$0.014 in accuracy, with a corresponding 95\% CI of $\pm$0.017. At 8B and 14B scales, \ours consistently outperforms advanced LLM agents, agentic frameworks, and matched text and latent baselines. Best performance is \highlightbox{\textbf{highlighted}}.
}
\vspace{-8pt}
\label{tab:main_agentic}
\setlength{\tabcolsep}{3pt}
\resizebox{\textwidth}{!}{
\begin{tabular}{l cccc|cccc|cccc|cc|cc}
\toprule
\multirow{2}{*}{\textbf{Method}}
  & \multicolumn{4}{c|}{\textbf{GAIA}}
  & \multicolumn{4}{c|}{\textbf{WebWalker}}
  & \multicolumn{4}{c|}{\textbf{HLE}}
  & \multicolumn{2}{c|}{\textbf{SealQA}}
  & \multicolumn{2}{c}{\textbf{DSQA}} \\
\cmidrule(lr){2-5} \cmidrule(lr){6-9} \cmidrule(lr){10-13} \cmidrule(lr){14-15} \cmidrule(lr){16-17}
  & Lv.1 & Lv.2 & Lv.3 & Avg.
  & Easy & Med. & Hard & Avg.
  & NS & CE & SF & Avg.
  & S-0 & Hard
  & Acc. & F1 \\
\midrule
\noalign{%
  \begingroup
  \color{gray!10}%
  \vskip-\belowrulesep
  \hrule height\belowrulesep
  \endgroup
}
\rowcolor{gray!10}
\multicolumn{17}{l}{%
  \raisebox{0.5pt}[\height][\depth]{\textbf{Frontier LLM Agents}}%
} \\
QwQ-32B & 30.9 & 6.5 & 5.2 & 18.9 & 7.5 & 2.1 & 4.6 & 4.3 & 11.5 & 7.3 & 5.2 & 9.6 & 1.8& 10.6 & 11.3 & 23.9 \\
GLM-5.3-Flash & 38.5 & 28.8 & 8.3 & 30.1 & 8.9 & 8.3 & 9.9 & 9.0 & 13.3 & 14.7 & 16.8 & 14.2 & 9.0 & 22.0 & 16.0 & 38.5 \\
DeepSeekV4.1-Flash & 43.6 & 42.3 & 16.7 & 39.8 & 22.2 & 16.7 & 12.7 & 16.5 & 28.5 & 25.3 & 30.5 & 28.4 & 14.4 & 33.5 & 19.3 & 38.6 \\
Kimi-K3 & 46.2 & 42.3 & 25.0 & 41.7 & 11.1 & 11.9 & 7.0 & 10.0 & 29.4 & 32.0 & 37.9 & 31.4 & 11.7 & 28.0 & 21.7 & 46.0 \\
GPT-5.6-Luna & 48.7 & 32.7 & 25.0 & 37.9 & 15.6 & 8.3 & 11.3 & 11.0 & 30.9 & 25.3 & 25.3 & 29.0 & 11.7 & 29.1 & 23.3 & 49.4 \\

\midrule
\noalign{%
  \begingroup
  \color{gray!10}%
  \vskip-\belowrulesep
  \hrule height\belowrulesep
  \endgroup
}
\rowcolor{gray!10}
\multicolumn{17}{l}{%
  \raisebox{0.5pt}[\height][\depth]{\textbf{Agentic Frameworks}}%
} \\
Search-o1 & 46.2 & 23.1 & 0.0 & 29.1 & 11.1 & 21.4 & 16.9 & 17.5 & 6.4 & 4.0 & 10.5 & 6.8 & 3.6 & 11.8 & 12.3 & 25.3 \\
WebThinker & 48.7 & 26.9 & 8.3 & 33.0 & 13.3 & 23.8 & 18.3 & 19.5 & 7.0 & 4.0 & 9.5 & 7.0 & 4.5 & 12.6 & 11.7 & 24.7 \\
ReAct & 48.7 & 25.0 & 8.3 & 32.0 & 11.1 & 20.2 & 12.7 & 15.5 & 5.8 & 5.3 & 10.5 & 6.6 & 7.2 & 15.0 & 14.6 & 32.3 \\

\midrule
\noalign{%
  \begingroup
  \color{gray!10}%
  \vskip-\belowrulesep
  \hrule height\belowrulesep
  \endgroup
}
\rowcolor{gray!10}
\multicolumn{17}{l}{%
  \raisebox{0.5pt}[\height][\depth]{\textbf{8B-level Latent \& Text Methods}}%
} \\
Qwen3-8B (thinking) & 28.1 & 15.4 & \underline{16.7} & 20.4 & 0.0 & 2.4 & 2.8 & 2.0 & 3.9 & 2.7 & 8.4 & 4.6 & 2.7 & 5.9	& 3.3 & 14.6 \\
Coconut & 7.7 & 1.9 & 0.0 & 3.9 & 0.0 & 1.2 & 0.0 & 0.5 & 1.5 & 0.0 & 1.1 & 1.2 & 0.0 & 1.6 & 1.7 & 6.1 \\
SoftThinking & 23.1 & 9.6 & 8.3 & 14.6 & 0.0 & 1.2 & 1.4 & 1.0 & 3.6 & 2.7 & 6.3 & 4.0 & 0.0 & 3.9 & 2.3 & 11.2 \\
GRPO & 48.7 & 25.0 & 8.3 & 32.0 & 28.9 & 32.1 & \underline{28.2} & 30.0 & \underline{7.9} & 4.0 & 10.5 & 7.8 & 15.3 & 23.6 & 17.3 & 33.8 \\
ARPO & \gaincell{53.9} & \underline{32.7} & \underline{16.7} & \underline{38.8} & \underline{31.1} & \underline{35.7} & \underline{28.2} & \underline{32.0} & 7.3 & \underline{6.7} & \underline{15.8} & \underline{8.8} & \underline{18.9}	& \underline{26.0} & \underline{18.0} & \underline{37.2} \\
\textbf{\ours (8B)} & \underline{51.3} & \gaincell{44.2} & \gaincell{25.0} & \gaincell{43.7} & \gaincell{48.9} & \gaincell{45.2} & \gaincell{35.2} & \gaincell{42.5} & \gaincell{14.2} & \gaincell{8.0} & \gaincell{20.0} & \gaincell{14.4} & \gaincell{22.5} & \gaincell{35.8} & \gaincell{20.0} & \gaincell{40.7} \\[-0.4ex]

\midrule
\noalign{%
  \begingroup
  \color{gray!10}%
  \vskip-\belowrulesep
  \hrule height\belowrulesep
  \endgroup
}
\rowcolor{gray!10}
\multicolumn{17}{l}{%
  \raisebox{0.5pt}[\height][\depth]{\textbf{14B-level Latent \& Text Methods}}%
} \\
Qwen3-14B (thinking) & 33.3 & 13.5 & 0.0 & 19.4 & 6.7 & 2.4 & 4.2 & 4.0 & 5.5 & 6.7 & 11.6 & 6.8 & 3.6 & 7.9 & 4.0 & 15.1 \\
Coconut & 10.3 & 5.8 & 0.0 & 6.8 & 2.2 & 2.4 & 1.4 & 2.0 & 3.9 & 4.0 & 5.3 & 4.2 & 0.0 & 2.8 & 2.0 & 7.5 \\
SoftThinking & 23.1 & 11.5 & 8.3 & 15.5 & 6.7 & 3.6 & 7.0 & 5.5 & 3.9 & 5.3 & 9.5 & 5.2 & 2.7 & 6.3 & 5.7 & 15.4 \\
GRPO & 51.3 & 34.6 & 0.0 & 36.9 & 35.6 & 42.9 & 35.2 & 38.5 & 7.9 & 6.7 & 12.6 & 8.6 & \underline{17.1} & \underline{24.4} & \underline{18.6} & \underline{43.9} \\
ARPO & \gaincell{56.4} & \underline{40.4} & \underline{16.7} & \underline{43.7} & \underline{40.0} & \underline{44.1} & \underline{36.6} & \underline{40.5} & \underline{10.3} & \underline{10.7} & \underline{13.7} & \underline{10.0} & \underline{17.1} & 22.8 & 16.3 & 38.2 \\
\textbf{\ours (14B)} & \underline{53.8} & \gaincell{55.8} & \gaincell{33.3} & \gaincell{52.4} & \gaincell{48.9} & \gaincell{52.4} & \gaincell{46.5} & \gaincell{49.5} & \gaincell{16.1} & \gaincell{13.3} & \gaincell{17.9} & \gaincell{16.0} & \gaincell{24.3} & \gaincell{37.8} & \gaincell{21.7} & \gaincell{46.8} \\[-0.4ex]
\bottomrule
\end{tabular}}
\vspace{-4pt}
\end{table*}

\textbf{Tasks and Datasets.} We conduct comprehensive evaluations of \ours on 11 benchmarks across various domains:
(i) \textit{Mathematical Reasoning}, including AIME 25 \citep{aime25}, AIME 26 \citep{aime26}, and MATH500 \citep{math500};
(ii) \textit{Knowledge-Intensive Reasoning}, including GPQA-Diamond \citep{gpqa}, MedQA \citep{medqa}, and MuSiQue \citep{musique}; 
and (iii) \textit{Agentic Reasoning}, including General AI Assistants (GAIA)  \citep{mialon2024gaia}, Humanity's Last Exam (HLE) \citep{hle}, WebWalker \citep{webwalker}, SealQA \citep{pham2026sealqa}, and DeepSearchQA (DSQA) \citep{dsqa}. We adopt the standard evaluation metric for each dataset. Additional benchmark details are in Appendix \ref{app:benchmark}.

\textbf{Baselines.}
For baseline comparisons, we evaluate \ours against
(i) \textit{Text-based Agentic Methods}, covering standard text SFT (with tools) and multi-turn tool-integrated RL methods such as GRPO \citep{shao2024deepseekmath} and ARPO \citep{dong2025agentic};
(ii) \textit{Latent Reasoning Methods}, including SoftThinking \citep{zhang2025soft}, Coconut \citep{hao2024training}, and SoftCoT \citep{xu-etal-2025-softcot};
(iii) \textit{Frontier LLM Agents}, including GPT-5.6-Luna \citep{openai_gpt56_luna}, Kimi-K3 \citep{kimik3}, GLM-5.3-Flash \citep{glm5}, DeepSeek-V4.1-Flash \citep{deepseekv41flash}, QwQ-32B \citep{qwq32b};
and (iv) \textit{Agentic Frameworks}, including ReAct \citep{yao2023react}, Search-o1 \citep{search-o1}, and WebThinker \citep{li2025webthinker}.
Additional baseline details are provided in Appendix~\ref{app:baseline}.

\textbf{Models and Training Setups.} We instantiate \ours with Qwen3-Base backbones \citep{yang2025qwen3} at three scales of 4B, 8B, and 14B. We then apply HybridSFT with 54K hybrid trajectories curated through our automatic annotation pipeline, followed by HybridRL with 12K deep reasoning examples \citep{yu2025dapo} and 1K agentic search examples \citep{dong2025agentic}. During RL, we set rollout batches of 512 trajectories (\(64\times8\)), followed by one optimization pass over 8 minibatches.

\textbf{Implementation Details.} \ours interacts with external environments primarily through two types of tools: a Python interpreter within sandbox environments and a live search engine (e.g., Tavily). We cap tool usage at 10 calls per trajectory and retrieve the top 10 snippets for each search query.
As current serving engines do not support multi-turn hybrid generation,
we extend the training and inference infrastructure to enable text, latent, and tool-calling transitions and persistent KV working-memory transfer.
RL rollouts use a sampling temperature of 1.0, and evaluation uses temperature 0.6, top-p 0.95, and top-k 20.
We conduct experiments on H100 GPUs using five independent seeds.
Additional training/inference details and hyperparameter setups are provided in Appendix \ref{app:imple_setup}.

\begin{table*}[!t]
\vspace{-5pt}
\centering
\caption{
\textbf{Main Results of \ours on Math and Knowledge-Intensive Search.} We evaluate performance across both SFT and RL stages. Small-size benchmarks like AIME25/26 are reported on avg@16 accuracy. Across all tasks, \ours yields an average std. of $\pm$0.031 (math) and $\pm$0.018 (search), with corresponding 95\% CI of $\pm$0.038 and $\pm$0.022.
Under the same training budgets and model setups, HybridSFT and HybridRL improve both text-based and latent reasoning methods.
}
\vspace{-8pt}
\label{tab:main_others}
\setlength{\tabcolsep}{4pt}
\resizebox{0.9\textwidth}{!}{%
\begin{tabular}{lcccccc}
\toprule
\textbf{Method}
& \textbf{AIME25} & \textbf{AIME26} & \textbf{MATH500}
& \textbf{GPQA-D} & \textbf{MedQA} & \textbf{MuSiQue} \\

\midrule
\noalign{%
  \begingroup
  \color{gray!10}%
  \vskip-\belowrulesep
  \hrule height\belowrulesep
  \endgroup
}
\rowcolor{gray!10}
\multicolumn{7}{l}{%
  \raisebox{0.5pt}[\height][\depth]{\textbf{Qwen3-8B-Base}}%
} \\
Coconut & 13.1 & 9.6 & 51.4 & 26.8 & 61.7 & 14.0 \\
SoftCoT & 16.5 & 17.9 & 63.6 & 28.9 & 62.0 & 19.5 \\
Standard Text SFT & 30.0 & 28.3 & 78.2 & 32.3 & 69.7 & 23.5 \\
\textbf{\ours (SFT)} & 30.8 & 29.2 & 78.4 & 31.8 & 70.0 & 26.5 \\
\midrule
GRPO & 36.3 & 39.4 & 82.0 & 43.9 & 74.3 & 27.5 \\
ARPO & \underline{37.1} & \underline{41.5} & \underline{82.8} & \underline{46.5} & \underline{74.7} & \underline{28.5} \\
\textbf{\ours (RL)} & \gaincell{39.2} & \gaincell{43.3} & \gaincell{84.6} & \gaincell{48.0} & \gaincell{75.3} & \gaincell{32.0} \\[-0.4ex]

\midrule
\noalign{%
  \begingroup
  \color{gray!10}%
  \vskip-\belowrulesep
  \hrule height\belowrulesep
  \endgroup
}
\rowcolor{gray!10}
\multicolumn{7}{l}{%
  \raisebox{0.5pt}[\height][\depth]{\textbf{Qwen3-14B-Base}}%
} \\
Coconut & 16.3 & 15.2 & 61.2 & 29.3 & 65.3 & 17.5 \\
SoftCoT & 22.7 & 23.8 & 74.0 & 36.4 & 67.7 & 20.0 \\
Standard Text SFT & 30.8 & 31.7 & 83.8 & 43.9 & 78.0 & 28.0 \\
\textbf{\ours (SFT)} & 32.5 & 31.7 & 84.2 & 45.5 & 78.7 & 29.5 \\
\midrule
GRPO & 40.8 & 42.5 & \underline{87.4} & 52.5 & \underline{81.0} & \underline{32.0} \\
ARPO & \underline{42.1} & \underline{43.3} & 87.0 & \underline{54.0} & 80.3 & 31.5 \\
\textbf{\ours (RL)} & \gaincell{43.3} & \gaincell{45.8} & \gaincell{88.2} & \gaincell{55.1} & \gaincell{81.7} & \gaincell{34.5} \\[-0.4ex]
\bottomrule
\end{tabular}}
\vspace{-4pt}
\end{table*}

\begin{figure}[!t]
    \centering
    \includegraphics[width=\linewidth]{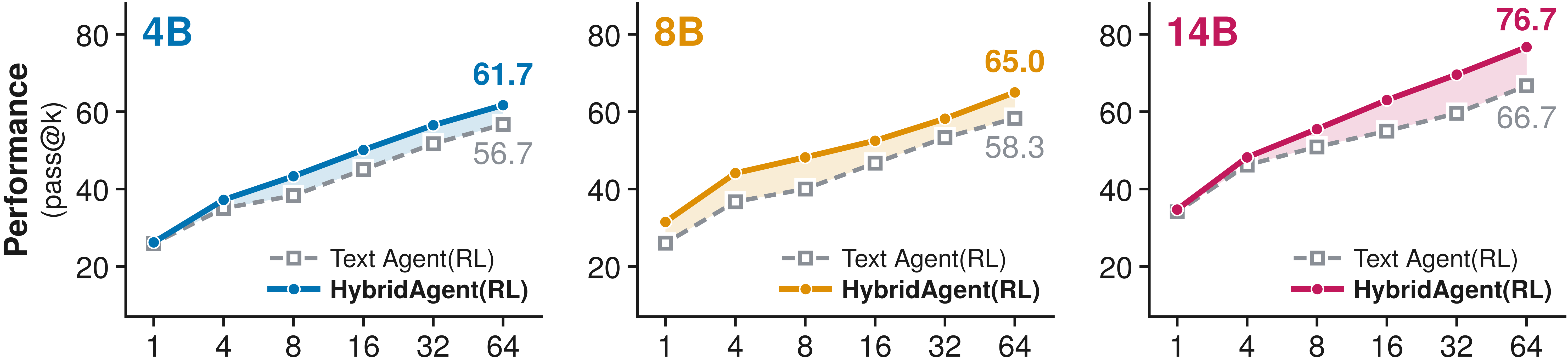}
    \caption{\textbf{Scalability of \ours at Test-time.} We evaluate pass@k ($4,8,16,32,64$) performance on 100 hard math and science questions and compare with text agents trained with GRPO.
    }
    \label{fig:passk}
    \vspace{-8pt}
\end{figure}

\vspace{-4pt}
\subsection{Scaling Performance via Hybrid Reasoning}

\textbf{Main Results.}
Tables \ref{tab:main_agentic} and \ref{tab:main_others} report the overall performance of \ours across 11 benchmarks spanning math, science, medicine, and knowledge search domains. 
On challenging agentic reasoning benchmarks, \ours delivers competitive performance over larger frontier LLM agents and specialized tool-use frameworks, with \ours (14B) improving over the strongest agentic baseline by an average of 17.4\%.
In addition, under matched model sizes and training budgets, HybridSFT and HybridRL demonstrate consistent training effectiveness and improve over corresponding text/latent baselines by 18.5\% at the SFT stage and 10.4\% at the RL stage.
Together, \ours combines the strengths of text-based agentic tool use with faster and more expressive latent reasoning, allowing the agent to interact externally while achieving deeper thinking depth internally.

\begin{figure}[!t]
    \centering
    \includegraphics[width=\linewidth]{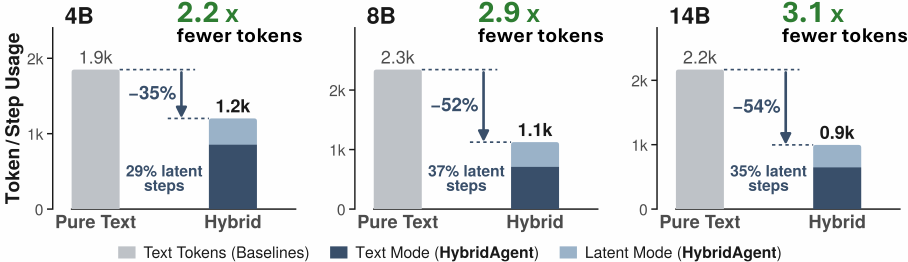}
    \vspace{-12pt}
    \caption{
    \textbf{Efficient Hybrid Generation of \ours across Model Scales.} \ours reduces both text-token usage and total generation steps compared with matched text-based LLM agents.
    }
    \label{fig:token_eff}
    \vspace{-8pt}
\end{figure}

\textbf{Scaling Hybrid Reasoning at Test Time.}
To further examine the test-time scalability of our hybrid reasoning paradigm, we increase inference compute budgets through additional sampling and evaluate pass@k performance.
Figure \ref{fig:passk} shows that \ours exhibits a clear upward trend as sampling increases across all model scales, while benefiting more from increased test-time compute than agents with text-only RL. Additional 4B-scale experiments and case studies are in Appendix~\ref{app:additional_exps} and \ref{app:case_probe}.

\vspace{-6pt}
\subsection{Efficiency Advantage on Hybrid Generation}
\vspace{-3pt}

\begin{wrapfigure}{r}{0.4\linewidth}
\vspace{-10pt}
    \centering
    \includegraphics[width=\linewidth]{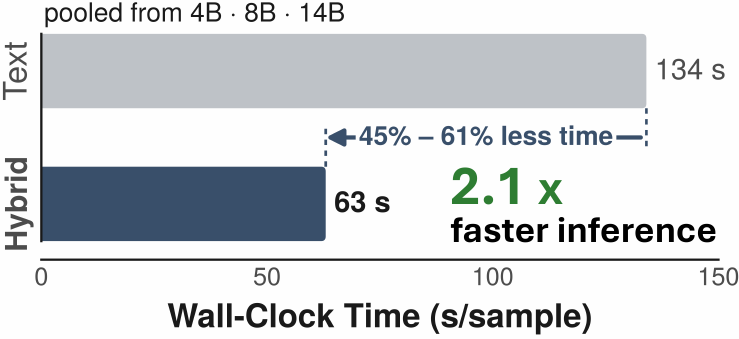}
    \caption{Inference time speedup of \ours across all model scales.}
    \label{fig:time_eff}
\vspace{-5pt}
\end{wrapfigure}
Beyond performance effectiveness, a key architectural advantage of \ours is that latent mode enables efficient generation without costly explicit token decoding.
We evaluate this property empirically along three complementary dimensions: (i) \textit{token throughput}, (ii) \textit{wall-clock inference time}, and (iii) \textit{actual computation cost}.
As shown in Figures \ref{fig:token_eff} and \ref{fig:time_eff}, \ours adaptively balances its hybrid reasoning across text and latent modes by using $2.7\times$ fewer text tokens, reducing total generation steps by up to 54\%, and achieving $2.1\times$ faster inference than text-based baselines.
Together with the cost analyses in Table \ref{tab:inference_cost} (see Appendix \ref{app:cost_analyses}), \ours generates more concise and refined reasoning trajectories, reducing inference overhead and cost.

\vspace{-10pt}
\section{In-depth Analyses}
\label{app:in_depth_analyses}

\textbf{RoadMap.}
We further analyze \ours from 3 complementary perspectives to assess the method design and better understand its hybrid reasoning behavior. 
\textbf{(i) Training Dynamics}: We study how self-conditioning in HybridSFT improves reasoning performance (Table \ref{tab:hybrid_SFT_compare}) and latent interpretability (Figure \ref{fig:HybridSFT_probe_all}), and how HybridRL shifts more reasoning into latent mode while strengthening tool interaction (Figure \ref{fig:training_dynamics}).
\textbf{(ii) Latent-Mode}: We intervene on latent content, ordering, and reasoning depth to examine \ours's causal reliance on latent thoughts (Table \ref{tab:causal_interventions}), and ablate \link architectures to assess the benefits of residual alignment (Table \ref{tab:hybridlink_design}).
\textbf{(iii) Text-Mode}: We evaluate text-based agentic tool use (Figure \ref{fig:tool_use_analyses}) and auditability of \ours (Table \ref{tab:monitorability}) to verify that text-mode supports reliable environment interaction and remains informative for external monitoring/debugging.

\vspace{-8pt}
\subsection{Training Dynamics on HybridSFT and HybridRL}
\label{app:ana_training}
\vspace{-10pt}

\begin{figure}[!ht]
\centering
\refstepcounter{table}\label{tab:hybrid_SFT_compare}
\begin{minipage}[c]{0.56\linewidth}
\centering
\small
\setlength{\tabcolsep}{3.5pt}
\resizebox{\linewidth}{!}{%
\begin{tabular}{lccc}
\toprule
\textbf{Method} & \textbf{AIME25} & \textbf{AIME26} & \textbf{MATH500} \\
\midrule
Coconut & 16.3 & 15.2 & 61.2 \\
Standard Text SFT & 30.8 & \textbf{31.7} & 83.8 \\
\midrule
Qwen3-14B-Base & 4.9 & 3.0 & 47.0 \\
+ Format Activation & 20.6 & 19.8 & 62.4 \\
+ \link Activation & 23.8 & 22.5 & 73.8 \\
\textbf{+ Self-Condition (HybridSFT)} & \textbf{32.5} & \textbf{31.7} & \textbf{84.2} \\
\bottomrule
\end{tabular}
}
\end{minipage}\hfill
\begin{minipage}[c]{0.42\linewidth}
\centering
\includegraphics[width=\linewidth]{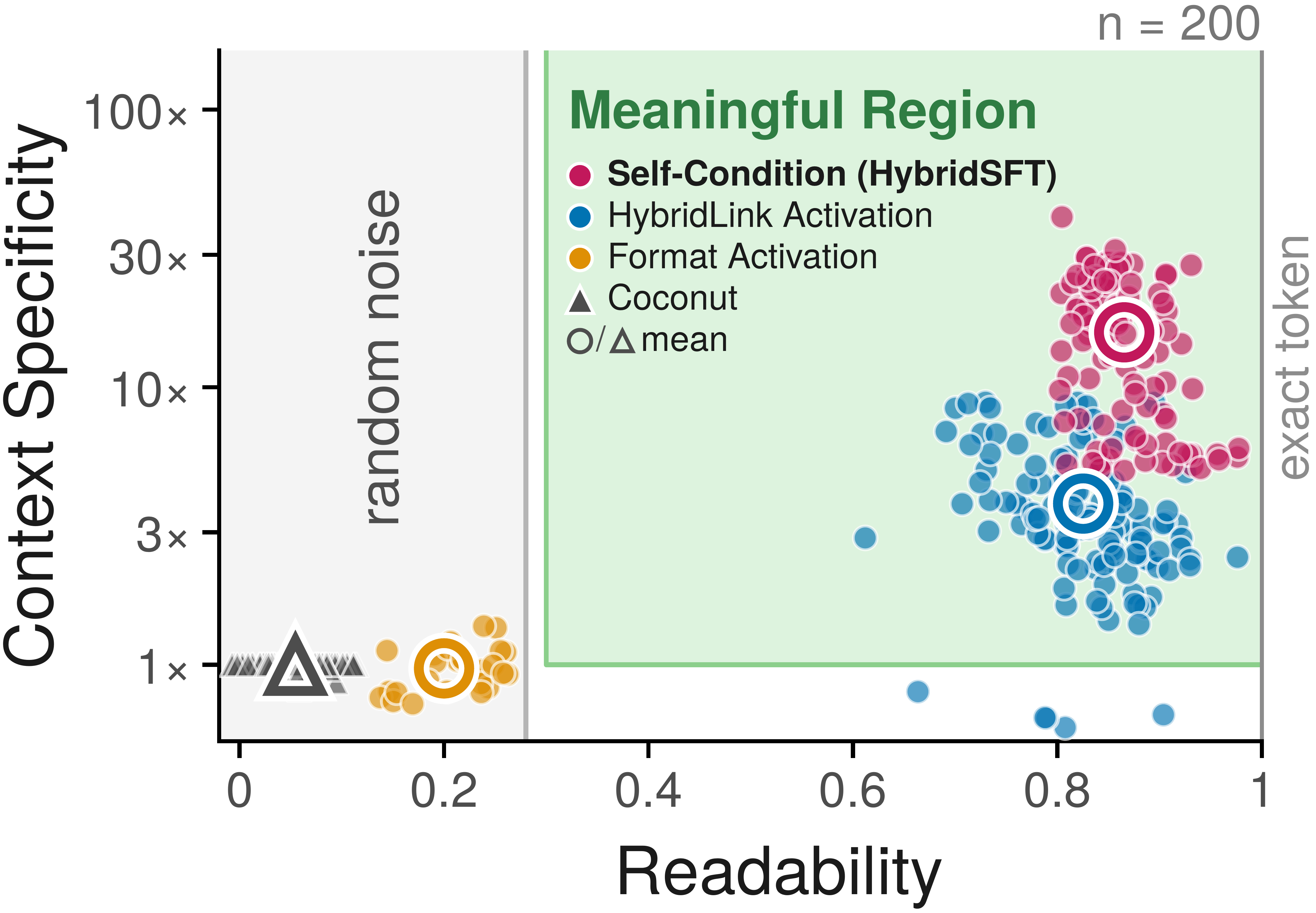}
\end{minipage}
\renewcommand{\figurename}{Table~\thetable\ and Figure}
\caption{
\textbf{Effectiveness and Interpretability of Self-Conditioned HybridSFT.}
\textit{Left:} We compare successive HybridSFT stages with latent and text training baselines using the same Qwen3-14B-Base backbone.
\textit{Right:} We evaluate the readability and context-specificity across 200 downstream questions, showing how self-conditioning improves latent transparency.
}
\label{fig:HybridSFT_probe_all}
\end{figure}

\begingroup
\setlength{\emergencystretch}{1em}
\textbf{Self-Conditioning Enables Effective and Interpretable HybridSFT.}
To understand why self-conditioned training is necessary for \ours during SFT, we compare HybridSFT against variants with only format and \link activations as well as latent SFT baselines like Coconut, from both performance and interpretability perspectives.
As shown in Table~\ref{tab:hybrid_SFT_compare}, although all SFT variants improve the original shared Qwen3-14B-Base backbone, activation alone remains insufficient. Self-conditioned supervision further improves performance by an average of 30.5\% over \link activation, bringing HybridSFT to parity with standard text SFT across all three math benchmarks.
\par\endgroup

We further probe the semantics of the latent thoughts produced by each method over 200 questions spanning all downstream domains.
Figure~\ref{fig:HybridSFT_probe_all} reveals a clear progression, where \link activation first makes latent contents more interpretable, and the additional self-conditioned training sharpens the latent readability and semantic informativeness.
We then take a closer look to examine individual test instances through detailed case studies in Appendix \ref{app:case_probe}. 
Specifically, methods with low faithfulness and specificity generally produce largely unreadable artifacts of the embedding geometry, whereas the complete HybridSFT yields clear and semantically coherent latent readouts.
Overall, these analyses reveal that self-conditioned training is essential to HybridSFT for making the latent reasoning traces in \ours fully transparent and directly interpretable.

\begin{figure}[!h]
    \centering
    \includegraphics[width=\linewidth]{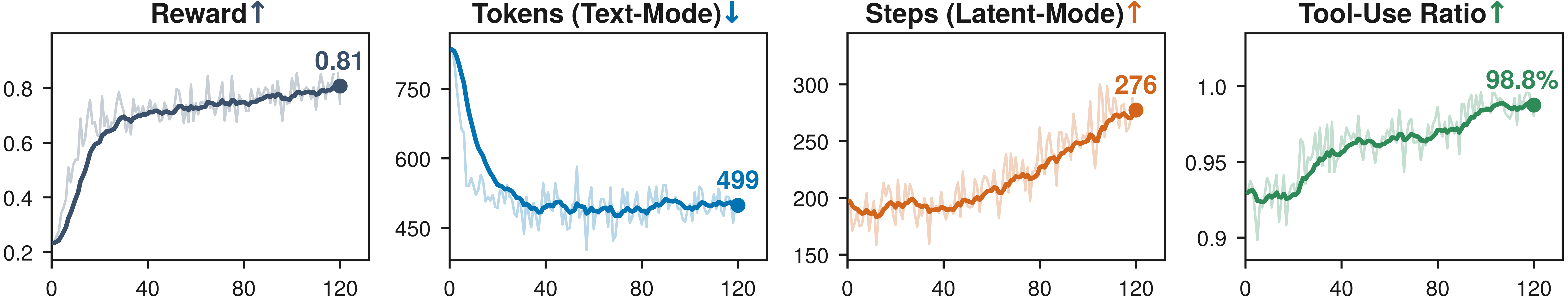}
    \caption{\textbf{Training Dynamics of HybridRL.} As training progresses, \ours shifts more generation from text to latent mode while increasing reward and tool use. These trends reflect how \ours adaptively learns to allocate its reasoning and tool calling efforts across the two modes.
    }
    \label{fig:training_dynamics}
\end{figure}

\textbf{Training Dynamics of HybridRL.} We analyze \ours{}'s training behaviors to examine how it dynamically shifts and allocates reasoning effort across text/latent modes under HybridRL.
In Figure \ref{fig:training_dynamics}, as the training step increases, our 14B agent progressively allocates more generation to latent mode while simultaneously strengthening its tool-invocation capability.
Through training, \ours adaptively learns to better utilize its two modes, shifting more reasoning into faster latent thinking while selectively exposing only key steps for external auditing and tool use through the text interface.

\begin{wraptable}{r}{0.43\textwidth}
\centering
\vspace{-10pt}
\caption{Ablation on HybridRL rewards.}
\vspace{-5pt}
\label{tab:reward_hybridRL}
\small
\resizebox{\linewidth}{!}{%
\begin{tabular}{@{}lccc@{}}
\toprule
\textbf{Reward Design} & MATH500 & GAIA & WebWalker \\
\midrule
HybridRL (8B) & \textbf{84.6} & \textbf{43.7} & \textbf{42.5} \\
\midrule
\quad w/o $r_{\text{correct}}$ & 79.2 & 35.9 & 37.5 \\
\midrule
\quad w/o $r_{\text{format}}$ & 80.6 & 40.7 & 40.5 \\
\midrule
\quad w/o $r_{\text{tool}}$ & 84.2 & 38.8 & 41.0 \\
\bottomrule
\end{tabular}%
}
\vspace{-8pt}
\end{wraptable}
\textbf{Reward Shaping in HybridRL.}
We ablate each reward component in \eqref{eq:rl_reward} using \ours (8B), removing one term at a time under otherwise identical setups.
In Table~\ref{tab:reward_hybridRL}, all three components provide complementary supervision. Excluding $r_{\text{correct}}$ causes the largest overall degradation, reducing performance by 6.4\%-17.8\%, while removing $r_{\text{format}}$ consistently lowers performance by 4.7\%-6.9\%.
The benefit of $r_{\text{tool}}$ is most pronounced on agentic benchmarks such as GAIA and WebWalker. Since the reward is granted only to trajectories that invoke both Python and search tools, removing $r_{\text{tool}}$ has a minor effect on MATH500, where Python alone is typically sufficient, but reduces GAIA and WebWalker by 11.2\% and 3.5\%.
Overall, each reward component provides complementary supervision for answer correctness, format structure, and effective multi-tool use.

\subsection{Causal Reliance and \link Design in Latent Mode}
\label{app:ana_latent}

\newlength{\interventiontablewidth}
\begin{table}[!h]
\centering
\caption{\textbf{Causal Interventions on \ours Latent Mode.}
We report MATH500 accuracy (\%) for 8B and 14B scales under different interventions on latent-mode thoughts generated by \ours.
Performance drops substantially when latent semantics or structure are disrupted, indicating that \ours causally relies on semantically meaningful latent-mode thoughts.
Here, $v$ denotes the original latent vector and $v'$ denotes the intervened
counterpart. The relative perturbation strength $d$ is computed as
$\|v'-v\|/\|v\|$, averaged over latent steps, where $d=0$ means no intervention.
}
\label{tab:causal_interventions}
\small
\renewcommand{\arraystretch}{1.16}
\setlength{\tabcolsep}{4pt}
\setlength{\interventiontablewidth}{%
  \dimexpr\textwidth-10\tabcolsep\relax
}

\setlength{\aboverulesep}{0pt}
\setlength{\belowrulesep}{0pt}
\arrayrulecolor{black}
\color{auditInk}
\newcommand{\interventioncategory}[3]{%
  \rowcolor{#1}%
  \multicolumn{5}{l}{\rule{0pt}{2.3ex}\textbf{\textcolor{#2}{#3}}} \\
}

\begin{tabular}{
  >{\raggedright\arraybackslash}m{0.21\interventiontablewidth}
  >{\raggedright\arraybackslash}m{0.45\interventiontablewidth}
  >{\centering\arraybackslash}m{0.12\interventiontablewidth}
  >{\centering\arraybackslash}m{0.10\interventiontablewidth}
  >{\centering\arraybackslash}m{0.12\interventiontablewidth}
}
\specialrule{0.8pt}{0pt}{0pt}
\textbf{Latent-Mode Intervention}
& \textbf{Setup Description}
& \textbf{Strength $d$}
& \textbf{8B-Scale}
& \textbf{14B-Scale} \\
\specialrule{0.5pt}{0pt}{0pt}
\interventioncategory{auditMonitorFill}{auditMonitorText}{Noise Perturbation}
\rowcolor{auditMonitorFill}
Gaussian Noise ($\sigma=0.5$)
& Add perturbation with random noise of magnitude $0.5\|v\|$ before latent feedback.
& 0.5 & 72.6 & 75.0 \\
\specialrule{0.25pt}{0pt}{0pt}
\rowcolor{auditMonitorFill}
Gaussian Noise ($\sigma=2.0$)
& Add perturbation with random noise of magnitude $2.0\|v\|$ before latent feedback.
& 2.0 & 66.4 & 70.2 \\
\specialrule{0.5pt}{0pt}{0pt}
\interventioncategory{auditDebugFill}{auditDebugText}{Latent Replacement}
\rowcolor{auditDebugFill}
Random-token Replacement
& Replace $v$ with a uniformly sampled token embedding, rescaled to $\|v\|$.
& 1.6 & 41.2 & 44.6 \\
\specialrule{0.25pt}{0pt}{0pt}
\rowcolor{auditDebugFill}
Cross-problem Replacement
& Replace each latent span with the corresponding latents from a different partner problem.
& 1.4 & 44.8 & 48.6 \\
\specialrule{0.25pt}{0pt}{0pt}
\rowcolor{auditDebugFill}
Mean-latent Replacement
& Replace every $v$ with the mean latent vector from the original \ours trajectories.
& 1.0 & 53.6 & 58.4 \\
\specialrule{0.25pt}{0pt}{0pt}
\rowcolor{auditDebugFill}
Nearest-token Replacement
& Replace $v$ with its nearest token embedding, rescaled to $\|v\|$.
& 0.2 & 82.6 & 87.2 \\
\specialrule{0.5pt}{0pt}{0pt}
\interventioncategory{auditSufficiencyFill}{auditSufficiencyText}{Latent Permutation}
\rowcolor{auditSufficiencyFill}
Within-block Permutation
& Randomly permute latent vectors within each block; Preserve content but disrupt order.
& 0.8 & 73.2 & 76.4 \\
\specialrule{0.5pt}{0pt}{0pt}
\interventioncategory{auditCorrectionFill}{auditCorrectionText}{Latent Computation Restriction}
\rowcolor{auditCorrectionFill}
One-step Only for Latent Span
& Force each latent span to terminate after one generation step.
& -- & 72.2 & 75.0 \\
\specialrule{0.25pt}{0pt}{0pt}
\rowcolor{auditCorrectionFill}
Block Latent-Mode Entry
& Prevent entry into latent mode, forcing the agent to reason entirely in text.
& -- & 80.4 & 78.8 \\
\specialrule{0.5pt}{0pt}{0pt}
\rowcolor{auditTraceFill}
\textbf{\ours (ours)}
& No intervention; the agent feeds back its own generated latent vector $v$ at every latent step.
& \textbf{0.0} & \textbf{84.6} & \textbf{88.2} \\
\specialrule{0.8pt}{0pt}{0pt}
\end{tabular}
\arrayrulecolor{black}
\end{table}

\textbf{Does \ours Causally Rely on its Latent Mode Thoughts?}
To examine whether latent thoughts are functionally used during hybrid reasoning, we intervene on the latent vectors generated by 8B and 14B \ours on the MATH500 downstream task. We perturb their content, ordering, and reasoning duration through 4 different types of controlled interventions, including \textcolor{auditMonitorText}{\textbf{(i) noise perturbation}}, \textcolor{auditDebugText}{\textbf{(ii) latent replacement}}, \textcolor{auditSufficiencyText}{\textbf{(iii) latent permutation}}, and \textcolor{auditCorrectionText}{\textbf{(iv) latent computation restriction}}. Table~\ref{tab:causal_interventions} details the intervention setups and reports performance under each scenario.

As shown in the table, disrupting latent thoughts consistently degrades performance across both scales. Specifically, cross-problem and random-token replacement reduce accuracy by 44.9-47.0\% and 49.4-51.3\%, while disrupting their ordering or restricting each latent span to one step leads to 13.4-15.0\% degradation. Notably, cross-problem replacement causes substantially greater degradation than even stronger Gaussian perturbations, suggesting that performance depends on the semantic content of latent thoughts rather than perturbation magnitude alone.
These analysis results on causal interventions indicate that \ours functionally relies on the semantics, ordering, and sufficient reasoning depth of its latent thoughts during hybrid reasoning.

\textbf{\link Design.} We next validate the effectiveness of \link and its residual connection by comparing our adopted 2-layer residual design against 4 alternative architectures: (i) direct hidden-state reuse (i.e., $e_{t+1}$$=$$h_t$) without learned alignment, (ii) 1-layer projection, (iii) 1-layer projection with residual connection, and (iv) 2-layer projection without residual connection. We end-to-end train each \link variant with the 14B backbone under the same setup and evaluate across tasks.

\begin{wraptable}{r}{0.5\textwidth}
\centering
\vspace{-10pt}
\caption{\textbf{Effectiveness of \link.}
We compare accuracy across alternative architectural designs.}
\label{tab:hybridlink_design}
\small
\resizebox{\linewidth}{!}{%
\begin{tabular}{lccc}
\toprule
\textbf{\link Design} & MATH500 & MedQA & MuSiQue \\
\midrule
Identity (No Alignment) & 71.4 & 73.7 & 24.5 \\
\midrule
1-Layer & 83.6 & 79.7 & 28.5 \\
\midrule
1-Layer+Res. & 86.4 & 80.3 & 31.0 \\
\midrule
2-Layer & 85.8 & 81.0 & 31.5 \\
\midrule
\textbf{2-Layer+Res. (ours)}
& \textbf{88.2} & \textbf{81.7} & \textbf{34.5} \\
\bottomrule
\end{tabular}%
}
\end{wraptable}
As shown in Table~\ref{tab:hybridlink_design}, the full two-layer residual design of \link yields consistent gains across all three downstream benchmarks, with the residual branch benefiting both one- and two-layer projection networks.
Specifically, incorporating the residual connection contributes an average improvement of 4.3\% across different layer depths.
This observed trend aligns with our design motivation in Section~\ref{sec:mode_tran}, where the residual connection preserves the semantics of input last-layer hidden states while the learned projection aligns these representations with the input embedding space.
Together, these two components make \link a more effective and stable interface for hybrid training and generation.

\subsection{Tool Use and Reasoning Auditability in Text Mode}
\label{app:ana_text}

\begin{figure}[!t]
\centering
\includegraphics[width=\linewidth]{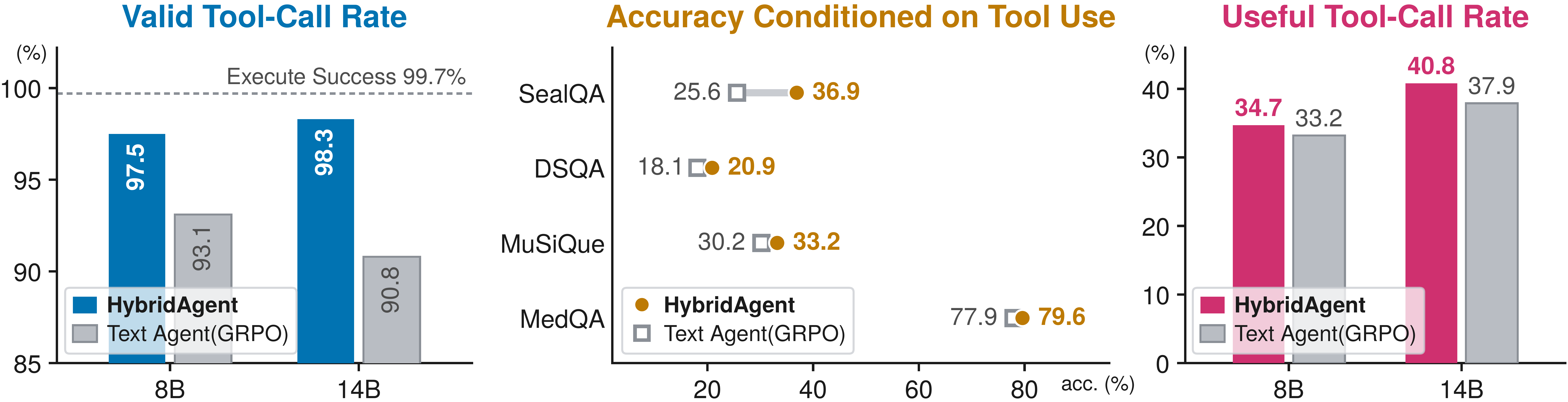}
\caption{\textbf{Agentic Tool-Use Analysis on \ours.}
\textit{Left}: Valid tool-call rate across 8B and 14B scales; "Execute Success" measures error-free execution among valid tool calls.
\textit{Middle}: Final-answer accuracy on questions where the agent performs at least one tool call.
\textit{Right}: Useful tool-call rate, measuring the fraction of tool calls whose execution results contain the final answer.
}
\label{fig:tool_use_analyses}
\end{figure}

\textbf{Does \ours Text-Mode Preserve Agentic Tool Use?}
In Figure~\ref{fig:tool_use_analyses}, we analyze agentic tool-calling capabilities by comparing \ours with normal LLM agents trained on GRPO, examining whether hybrid reasoning preserves reliable external interactions through its text mode.
Across 8B and 14B scales, \ours achieves higher valid tool-call rates while matching the 99.7\% execution success of valid calls, showing that hybrid reasoning preserves a reliable text-mode tool interface. More importantly, the gains extend beyond call validity, as \ours improves final-answer accuracy by 2.2\%-44.1\% on questions where tools are invoked. The 4.5\%-7.7\% higher useful tool-call rate further indicates that \ours not only invokes tools correctly, but also retrieves evidence that more directly contributes to solving the task.
These findings suggest that \ours preserves the reliable and effective agentic tool use of standard LLM agents through its text mode, while exposing fewer text tokens with interleaved latent reasoning.

\textbf{Does \ours Text-Mode Preserve External Auditability?}
A natural question of hybrid reasoning is whether shifting computation into latent mode might potentially compromise the external auditability provided by text. To examine this, we evaluate \ours against matched text agents on MATH500 and AIME25/26, covering 560 problems under identical setups. For each question, an external observer sees only the exposed text steps, tool calls, and tool outputs, while latent spans remain hidden behind their boundary markers. We evaluate four complementary dimensions following commonly adopted evaluation protocols.
\begin{itemize}[leftmargin=*]
    \item \textcolor{auditMonitorText}{\textbf{\textit{Monitoring}}}: Measure whether a blind external monitor can predict rollout correctness from the input problem and visible trace generated by each agent, following \cite{baker2025monitoring}. We report AUROC using GPT-4o-mini \citep{hurst2024gpt} as a weak monitor, which relies primarily on the visible trace rather than solving the problem independently, and GPT-5.6-Luna \citep{openai_gpt56_luna} as a strong monitor with more powerful task-solving capability.
    \item \textcolor{auditDebugText}{\textbf{\textit{Debugging}}}: Following first-error localization~\citep{zheng-etal-2025-processbench}, a judge model (GPT-5.6-Luna) is deployed to identify the first erroneous visible step from a failed rollout and produces a corrective hint, with the diagnosis verified against the reference ground-truth solution.
    \item \textcolor{auditCorrectionText}{\textbf{\textit{Correction}}}: Each failed problem is re-attempted using the corrective hint generated during debugging, and compared against a generic correctness-only feedback baseline \citep{huang2024large} that simply tells the agent its previous answer was incorrect and asks it to retry.
    \item \textcolor{auditSufficiencyText}{\textbf{\textit{Rationale Sufficiency}}}: Following sufficiency evaluation in \cite{deyoung-etal-2020-eraser}, a weak reader (GPT-4o-mini) predicts the agent's own answer from the exposed text steps alone after removing the final answer, measuring whether the visible trace retains sufficient decision-relevant information.
\end{itemize}

\begin{table}[!t]
\centering
\small
\caption{\textbf{Monitorability and Debuggability through \ours Text-Mode.}
We compare \ours with matched same-size text agents to examine whether its exposed text remains sufficient for external monitoring, debugging, correction, and rationale-based inspection under hybrid reasoning. All comparisons are evaluated using paired bootstrap 95\% confidence intervals.}
\label{tab:monitorability}
\renewcommand{\arraystretch}{1.18}
\setlength{\tabcolsep}{6pt}
\setlength{\aboverulesep}{0pt}
\setlength{\belowrulesep}{0pt}
\arrayrulecolor{black}
\color{auditInk}
\resizebox{\linewidth}{!}{%
\begin{tabular}{llcccc}
\specialrule{0.8pt}{0pt}{0pt}
\multirow{2}{*}{\textbf{Metric}}
& \multirow{2}{*}{\textbf{Measurement}}
& \multicolumn{2}{c}{\textbf{8B Scale}}
& \multicolumn{2}{c}{\textbf{14B Scale}} \\
\cmidrule(lr){3-4} \cmidrule(lr){5-6}
&
& \textbf{\ours} & \textbf{Text Agent}
& \textbf{\ours} & \textbf{Text Agent} \\
\specialrule{0.5pt}{0pt}{0pt}
\rowcolor{auditTraceFill}
& Text-Mode tokens per instance
& \textbf{302} & 596 & \textbf{284} & 497 \\
\rowcolor{auditTraceFill}
& Latent-Mode steps per instance
& \textbf{170} & -- & \textbf{183} & -- \\
\rowcolor{auditTraceFill}
\multirow{-3}{*}{\textbf{\textcolor{auditTraceText}{Visible Trace}}}
& Key reasoning steps per 100 visible tokens $\uparrow$
& \textbf{0.72} & 0.57 & \textbf{0.86} & 0.67 \\
\specialrule{0.35pt}{0pt}{0pt}
\rowcolor{auditMonitorFill}
& Weak-monitor AUROC (MATH500) $\uparrow$
& \textbf{0.840} & 0.772 & \textbf{0.860} & 0.818 \\
\rowcolor{auditMonitorFill}
\multirow{-2}{*}{\textbf{\textcolor{auditMonitorText}{Monitoring}}}
& Strong-monitor AUROC (all tasks) $\uparrow$
& \textbf{0.961} & 0.935 & 0.989 & 0.990 \\
\specialrule{0.35pt}{0pt}{0pt}
\rowcolor{auditDebugFill}
& First error correctly localized (\%) $\uparrow$
& \textbf{99.2} & 94.2 & 98.1 & 100.0 \\
\rowcolor{auditDebugFill}
& Error diagnosis matches reference (\%) $\uparrow$
& \textbf{97.6} & 93.5 & \textbf{97.2} & 96.3 \\
\rowcolor{auditDebugFill}
\multirow{-3}{*}{\textbf{\textcolor{auditDebugText}{Debugging}}}
& Useful corrective hint / Answer leakage (\%) $\uparrow/\downarrow$
& 97.6 / 0.8 & \textbf{98.7} / \textbf{0.0} & \textbf{100.0} / 0.9 & 98.2 / 0.9 \\
\specialrule{0.35pt}{0pt}{0pt}
\rowcolor{auditCorrectionFill}
\textbf{\textcolor{auditCorrectionText}{Correction}}
& Fix-rate gain from targeted hint vs.\ generic retry ($\Delta$) $\uparrow$
& \textbf{+9.1} & +3.4
& \textbf{+5.4} & +3.5 \\
\specialrule{0.35pt}{0pt}{0pt}
\rowcolor{auditSufficiencyFill}
\textbf{\textcolor{auditSufficiencyText}{Sufficiency}}
& Answer-agreement gain from visible text only ($\Delta$) $\uparrow$
& \textbf{+0.25} & +0.17 & \textbf{+0.30} & +0.28 \\
\specialrule{0.8pt}{0pt}{0pt}
\end{tabular}%
}
\arrayrulecolor{black}
\end{table}

As shown in Table~\ref{tab:monitorability}, \ours exposes 43\%-49\% fewer text-mode tokens while performing 170-183 reasoning steps in latent mode, yet the visible trace contains 26\%-28\% more key reasoning steps per 100 tokens. This reduction preserves external monitorability, with AUROC remaining within 0.022 of matched text agents across both weak and strong monitors. The exposed text also remains debuggable, localizing the first error in 98\%-99\% of failed rollouts and yielding reference-consistent diagnoses in 97\%-98\% of cases with below 1\% answer leakage. Moreover, text-derived corrective hints produce significant fix-rate gains of 5.4 and 9.1 points at 14B and 8B, while rationale sufficiency remains comparable to or higher than text agents.
As shown by the analyses above, \ours uses its text mode to maintain a compact yet informative textual interface for external monitoring, debugging, and correction.
This provides an important safeguard for agent generation, facilitating human auditing and intervention while preserving efficient internal reasoning.

\section{Related Work}

\textbf{Latent Reasoning Beyond Text.}
Recent works have explored applying intermediate reasoning beyond discrete language tokens into continuous hidden spaces \citep{hao2024training, zhu2025reasoning, yu2026latent}, allowing models to perform internal computation without explicitly verbalizing every reasoning step \citep{geiping2026scaling,shen2026efficient,zhu2025survey}.
Existing approaches broadly realize this idea by inputting last-layer hidden states \citep{saunshi2025reasoning,tack2026llm}, compressing textual chains of thought into latent representations \citep{cheng2024compressed,zhang2026memgen}, or directly optimizing continuous reasoning trajectories \citep{tan2026think,chen2024language}.
These methods improve reasoning efficiency and depth but make direct inspection and external interaction difficult \citep{korbak2025chain,chen2026what}.
\ours builds on this direction by integrating directly inspectable latent reasoning as a native built-in mode within the agentic generation process, while preserving direct text-based auditing and tool interaction.

\textbf{Agentic Reasoning and Tool Use.}
Frontier LLM agents \citep{kimik3,glm5,deepseekv41flash} extend reasoning beyond static generation by interacting with external environments through tools \citep{shen2024llmtoolssurvey,tool_survey}, such as search engines \citep{jin2025search,zheng2025deepresearcher}, code interpreters \citep{feng2026retool,bianchi2026harnessing}, and APIs \citep{qin2024toolllm,patil2025the}.
A growing line of work further improves these capabilities through post-training over multi-turn tool-use trajectories \citep{dong2025tool,qian2025toolrl}, enabling agents to plan, act, observe, and revise their reasoning \citep{su2025toolorchestra,li2025flow}.
Despite this progress, existing agentic frameworks remain text-centric for both model-side reasoning and environment-side interaction, incurring large generation overhead from verbalizing every intermediate step \citep{cuadron2025danger,xu2025chain}.
\ours decouples these roles by interleaving continuous thoughts for internal reasoning and retaining text only where explicit interaction and monitoring are required.

We discuss additional related work and its relation to \ours in Appendix~\ref{app:related_work}.

\section{Discussion and Future Directions}

\textbf{Implications for Frontier Agent Safety and Secure Deployment.}
Frontier labs (e.g., OpenAI, Anthropic) provide API-based access to reasoning models while hiding the chain-of-thought from customers to protect intellectual property and limit information leakage, yet \citet{panfilov2026stealing} show that the encrypted traces returned to clients can be replayed into a sibling model and transcribed verbatim.
\ours places proprietary reasoning on a different footing. Latent thoughts exist only as backbone hidden states that are readable solely through the provider-owned \link (Section~\ref{sec:architecture}), so no text exists to leak. 
The same property aids monitoring. Because the latent trace is held by the provider rather than the client, providers can assess the safety of reasoning and actions over the full readout, including the portion users never see, while the agent benefits from faster inner thinking. Text mode is then reserved for what users legitimately need, such as tool calls, key decisions, and the final answer, preserving an auditable trace without exposing the full reasoning.

\textbf{Future Directions on Hybrid Agent Harness and Efficient Post-Training.}
Future work can build a new type of agent harness around \ours, where latent working memory is a first-class component. Context management can then operate on KV entries directly, workflows can schedule steps by mode, keeping tool calls and human-facing checkpoints in text while routine computation stays latent, and one agent's hybrid trajectory can serve as its message to another. Hybrid generation can also make post-training more efficient. Because latent thoughts remain auditable during rollout, an external critic can assign step-level credit to latent spans, branch promising trajectories, or terminate ones whose latent reasoning has already gone wrong, before any text or tool call is produced.

\section{Conclusion}

We introduce \ours, a unified hybrid reasoning framework that combines efficient latent reasoning with text-based inspection and tool interaction.
\ours enables multi-turn hybrid generation through both architectural designs and dedicated post-training algorithms on SFT and RL stages.
Across 11 benchmarks and 3 model scales, \ours achieves strong performance against advanced LLM agents and agentic frameworks spanning text-based and latent reasoning methods, while substantially reducing inference time, token usage, and computational cost.
Overall, \ours provides a scalable and efficient framework for next-generation agents that can reason beyond human language while retaining agentic interaction in environments and external inspectability.

\section*{Acknowledgments}

The authors thank the members of the Zou Lab and the iDEA-iSAIL Lab for helpful discussions and feedback. Jiaru also thanks Jiuding Sun, Weiyan Shi, Martijn Bartelds, Yongchan Kwon, and many others in the Stanford NLP Group and the Together Agent Team for their support and discussions.

\bibliographystyle{abbrvnat}
\nobibliography*
\begingroup
\hypersetup{urlcolor=black}
\bibliography{main}
\endgroup
\newpage 

\appendix
\renewcommand{\contentsname}{\Large Table of Contents}
{\hypersetup{linkcolor=black}
\tableofcontents
}

\addtocontents{toc}{\protect\setcounter{tocdepth}{2}}

\newpage
\textbf{\Large Appendix}

\section{Multi-turn Continuation over Working Memory}
\label{app:theorey_memory}

Following our notation in Section \ref{sec:pre}, denote the input sequence $E_{\le t}=[e_1,\ldots,e_t]$ and working memory $\mathcal M_t=[K_{\le t},V_{\le t}]$ at step $t$, where $K$ and $V$ collect the keys and values across all layers.

\begin{theorem}[Lossless Multi-Turn Working Memory]
\label{thm:kv}
Given the multi-turn generation process in Section \ref{sec:multi_turn_generate}, consider the fixed \ours $f_\theta$ and \link $\mathcal R$ under consistent inference settings.
At step $t$, let $h_{t+1}$ be the hidden output obtained by processing $e_{t+1}$ with the complete working memory $\mathcal M_t$. Then,
\begin{equation}
\label{eq:kv-equality}
\underbrace{h_{t+1}}_{\text{process }e_{t+1}\text{ using }\mathcal M_t}
=
\underbrace{f_\theta([E_{\le t};e_{t+1}])}_
{\text{apply the complete hybrid prefix}}.
\end{equation}
Thus, working memory preserves the text, latent, and tool observations for continued reasoning, without converting latent thoughts into text.
\end{theorem}

\begin{remark}
Theorem \ref{thm:kv} shows that retaining complete working memory preserves the effect of the full hybrid history on subsequent generation across turns.
This justifies resuming \ours after tool interactions without reconstructing earlier turns or converting latent thoughts into text.
\end{remark}

\begin{proof}
Fix $f_\theta$ and \link $\mathcal R$, and use the same deterministic causal operations, masks, and positional encoding, preserving all existing KV entries. Each cached forward run produces both a hidden output and an updated memory:
\[
\underbrace{(e_{t+1},\mathcal M_t)}_{\text{next input and preceding context}}
\xrightarrow{\text{cached forward run}}
\underbrace{(h_{t+1},\mathcal M_{t+1})}_{\text{hidden output and updated memory}}.
\]
The hidden output determines the next text distribution or latent input, while the updated memory provides the context for the following forward run. Here, layer and head indices are omitted for notation simplicity.
We compare a memory-based forward run with a full-prefix forward run under the same model and inference settings.

\textit{Base case ($t=1$).}
Both forward runs start with empty memory and the same input $e_1$. At the first layer, they produce the same queries, keys, values, and attention outputs. The remaining token-wise operations produce the same input to the next layer. Repeating this argument through all layers gives
\[
\mathcal M_1=[K_1,V_1].
\]
Both forward runs also produce the same final hidden output,
$h_1=f_\theta([e_1])$.

\textit{Induction hypothesis.}
Suppose that, after processing step $t$, the stored keys and values are equal to those computed by a full-prefix forward run:
\[
\underbrace{\mathcal M_t}_{\text{stored KV memory}}
=
\underbrace{[K_{\le t},V_{\le t}]}_
{\text{KV computed from the complete prefix }E_{\le t}}.
\]

\textit{Induction step ($t\to t+1$).}
Causality and unchanged positional encoding preserve the earlier keys and values when $e_{t+1}$ is appended. Both forward runs start the first layer with the same input $e_{t+1}$.

If their inputs to the current layer are equal, each head computes the same query, new keys and values. The earlier KV entries are equal by the induction hypothesis. With the same attention mask, each head therefore produces the same attention output. The remaining token-wise operations produce the same input to the following layer.

Repeating this implication through all layers gives the same final hidden output:
\[
h_{t+1}=f_\theta([E_{\le t};e_{t+1}]).
\]
It also gives the same new keys and values, $K_{t+1},V_{t+1}$. Appending them to the stored memory and substituting the induction hypothesis yields
\begin{align*}
\mathcal M_{t+1}
&=\mathcal M_t\oplus[K_{t+1},V_{t+1}]\\
&=[K_{\le t},V_{\le t}]\oplus[K_{t+1},V_{t+1}]\\
&=[K_{\le t+1},V_{\le t+1}].
\end{align*}
Thus, the updated memory satisfies the induction hypothesis at $t+1$, while the hidden-output equality establishes \eqref{eq:kv-equality}. By induction, at every input position, the cached forward run produces the same hidden output and updated KV memory as the full-prefix forward run.

At the end of each turn $j \in \{1,\ldots,n\}$, let $t=|x|+\sum_{i=1}^{j}|\tau^{(i)}|$ count all input slots
processed so far, including the initial context and all text,
latent, and tool-observation inputs. Starting from $\mathcal M^{(0)}=[K_{\le |x|},V_{\le |x|}]$, repeated concatenation gives
\begin{align*}
\mathcal M^{(j)}
&=\mathcal M^{(0)}\oplus[K_{|x|+1:t},V_{|x|+1:t}]\\
&=[K_{\le |x|},V_{\le |x|}]
  \oplus[K_{|x|+1:t},V_{|x|+1:t}]\\
&=[K_{\le t},V_{\le t}].
\end{align*}
Each new entry is appended exactly once, so turn boundaries do not change the memory used by subsequent forward runs.

Finally, the established hidden-state equality gives
\begin{align*}
\operatorname{Softmax}(W_{\mathrm{out}}h_t)
&=\operatorname{Softmax}(W_{\mathrm{out}}f_\theta(E_{\le t})),\\
\mathcal R(h_t)
&=\mathcal R(f_\theta(E_{\le t})).
\end{align*}
where we retain the corresponding hidden states and tool observations, while using
the same random draws for sampling and tool execution in both runs.
Then, both runs process the same next input with the same memory and reasoning mode, producing the same hidden output and updated KV memory. 
Repeating this argument step by step gives identical subsequent hybrid sequences under the shared random draws, and hence identical generation probabilities. 
Therefore, working memory preserves the effect of the complete hybrid history across turns, without reconstructing latent thoughts as text.

\end{proof}

\clearpage

\section{Design and Analyses of Hybrid Supervised Fine-tuning}
\label{app:hybridsft}

We expand on the method description of HybridSFT (Section \ref{sec:HybridSFT}), providing additional design justifications and implementation details for the overall training scheme.

\subsection{Why Self-Conditioned Training?}
\label{app:self_condition}

\begin{figure}[!h]
\centering
\includegraphics[width=\linewidth]{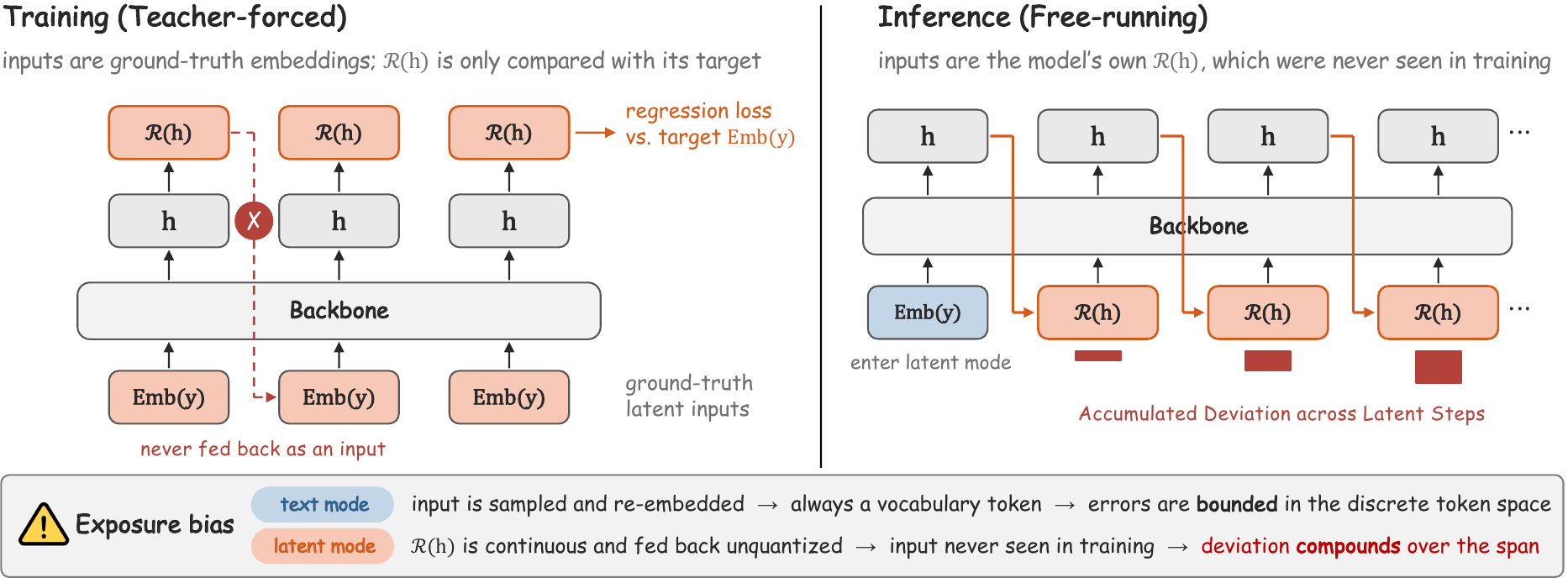}
\caption{\textbf{Train-Inference Mismatch in Latent Reasoning.} We illustrate the exposure bias that arises when teacher-forced training meets free-running inference. \textit{Left:} During activation, the backbone consumes ground-truth embeddings $\mathrm{Emb}(y)$ at latent positions, while \link outputs $\mathcal{R}(h)$ are only trained to match their targets and never fed back as inputs. \textit{Right:} At inference, every latent input is the model's own $\mathcal{R}(h)$, so deviations from the target embeddings accumulate across the span.}
\label{fig:exposure_bias}
\end{figure}

The format and \link activation provide a well-aligned initialization, but both are teacher-forced (Figure~\ref{fig:exposure_bias}, left). At every latent position, the backbone consumes the ground-truth embedding $\mathrm{Emb}(y)$, and \link is trained only to match its output $\mathcal{R}(h)$ to that target; the prediction itself is never fed back as an input. At inference, the agent instead runs free (Figure~\ref{fig:exposure_bias}, right). Once it enters latent mode, every subsequent input is its own $\mathcal{R}(h)$, which neither the backbone nor \link has been conditioned on. This train-inference mismatch is akin to \textit{exposure bias} \citep{ranzato2015sequence,bengio2015scheduled}, a standard issue in autoregressive sequence models that are trained on ground-truth prefixes but consume their own predictions at inference.

It is worth noting that exposure bias is present in both text and latent modes, since both feed the model's own outputs back as inputs.
In text mode (as in standard LLM SFT), each fed-back input is sampled and re-embedded, and thus always corresponds to a token in the discrete vocabulary. A prediction error, where the sampled token differs from the ground-truth target, only changes which token is fed back; the input representation itself never leaves the set seen in training, so errors cannot compound in the input space.
Latent mode, in contrast, has no such projection. $\mathcal{R}(h)$ is continuous and fed back unquantized, so even a small deviation from the target embedding becomes an input the backbone has never seen in training. This deviation propagates to the next hidden state and hence the next latent input, compounding over the tens to hundreds of steps in a latent span and gradually eroding the semantic coherence of the latent trace.

\textit{Self-conditioned HybridSFT} closes this gap by training both components on the model's own generated latent inputs. As in Figure~\ref{fig:hybrid_SFT}, the teacher-forced first pass applies \link to the last-layer hidden state at each latent position, yielding the model's own prediction of the next latent input. The second pass then replaces the ground-truth latent embeddings with these predictions and trains the backbone and \link jointly, so both parts learn under the same latent inputs they encounter at inference. 

Under this training scheme, the backbone learns to read, and \link learns to produce, latent inputs under the same conditions they encounter at inference. Thus, \ours is trained as an integrated whole rather than two separately activated components.

\subsection{Learning When to Enter and Exit Latent Mode}
\label{app:anchors}

As stated in Section~\ref{sec:architecture}, \ours transitions between modes through two special tokens, the entry anchor \texttt{<latent>} and the exit anchor \texttt{</latent>}. Within a latent span, these anchors are the only positions supervised by $\mathcal{L}_{\text{Backbone}}$, while every other latent position is supervised by $\mathcal{L}_{\text{\link}}$ alone. The two transitions are not symmetric, however. Entering latent mode is decided in text mode, where the backbone samples the next token from the vocabulary as usual, whereas exiting must be decided from a continuous latent state through the readout in Eq.~\ref{eq:latent-exit}. HybridSFT therefore treats the entry anchor as an ordinary token and gives the exit anchor dedicated supervision.

\textbf{Entry as a Text-Mode Decision.}
Since entering is a text mode decision, \texttt{<latent>} is simply a cross-entropy target, and the backbone learns from the preceding context when to switch. The anchor position then serves as the first latent step, with its last-layer hidden state passed through \link to produce the next latent input, so the span begins without any additional mechanism.

\textbf{Exit as a Latent-Mode Decision.} Exiting takes place in latent mode, where no token is sampled. By Eq.~\ref{eq:latent-exit}, a span closes when the \link's output $\mathcal{R}(h)$ is nearest to $\mathrm{Emb}(\texttt{</latent>})$, a decision depending on both components, where the backbone produces a hidden state that encodes the exit and \link maps it onto the anchor embedding. HybridSFT therefore supervises the last latent step of each span from both sides.
On the \link side, the regression target at this step is $\mathrm{Emb}(\texttt{</latent>})$, given slightly more weight than the span interior so that the single exit signal is not diluted. On the backbone side, \texttt{</latent>} is also a cross-entropy target predicted from the same hidden state, so the backbone itself learns to signal the exit.

\newpage
\section{Theoretical Analyses of Hybrid Reinforcement Learning}
\label{app:theorey_rl}

\subsection{Formulation of Hybrid Policy Factorization}

Following our notation in Algorithm \ref{alg:hybrid_rl}, denote $\tau=(z, y, o)$ as a hybrid trajectory, where $\mathcal T$ is the set of sampled text positions.
Here, $\pi_\theta(\tau\mid x)$ denotes the likelihood of the agent’s outputs along trajectory \(\tau\), given input $x$.
Meanwhile, tool observations are generated by the environment and serve as context for subsequent model decisions. We formalize the factorization of the hybrid policy below.

\begin{proposition}[Restatement of Hybrid Policy Factorization]
\label{prop:factorization_app}
Let $\tau$ be a hybrid trajectory generated by the policy $\pi_\theta$ (induced by $f_\theta$) with fixed $\mathcal{R}$. Then the hybrid policy factorizes as:
\begin{equation}
\label{eq:rl_factorization_app}
    \pi_\theta(\tau \mid x) \;=\; \prod_{t \in \mathcal{T}} \pi_\theta\bigl(y_t \mid x,\, \tau_{<t}\bigr),
\end{equation}
where $\mathcal{T}$ is the set of sampled text positions of trajectory $\tau$.
\end{proposition}

\begin{proof}
Fix $\theta$ and a realized trajectory $\tau$. At a latent continuation,
we have $\Pr(e_t=\mathcal R(h_{t-1})\mid x,\tau_{<t})=1$ with a fixed \link $\mathcal R$.
Deterministic reasoning mode transitions (without sampling) also have conditional mass $1$. The probability chain rule thus gives
\[
\pi_\theta(\tau\mid x)
=\left[\prod_{t\in\mathcal T}\pi_\theta(y_t\mid x,\tau_{<t})\right]
 \left[\prod_{\substack{t:\,\text{deterministic}\\\text{steps}}}1\right]
=\prod_{t\in\mathcal T}\pi_\theta(y_t\mid x,\tau_{<t}),
\]
where deterministic steps receive the probability of $1$.

\end{proof}

\begin{remark}
Proposition~\ref{prop:factorization_app} shows that the hybrid policy factorizes over sampled text positions, with each text-token probability conditioned on the complete preceding hybrid context, including text, latent thoughts, and tool observations.
Building on this factorization formulation, we show in the next subsection (Theorem \ref{thm:rl}) how exact rollout replay enables HybridRL's policy updates to assign credit to complete recorded hybrid trajectories.
\end{remark}

\subsection{Hybrid Policy Optimization}

Recall that we collect $\{\tau_i=(z_i,y_i,o_i)\}_{i=1}^{G}$, with group size $G\ge2$, using the rollout policy $\pi_{\theta_{\mathrm{old}}}$.
During replay, $\pi_\theta(y_{i,t}\mid x,\tau_{i,<t})$ scores the recorded prefix
with the current model across step $t$; $\pi_\theta(\tau_i\mid x)$ is the product
in~\eqref{eq:rl_factorization_app} under this same evaluation.
Let $r_i=r(\tau_i)$ be the reward, and $A_i$ be the fixed group-normalized advantage.
Let $\mathcal T_i$ denote the sampled text positions of $\tau_i$. Define
\begin{align}
\rho_{i,t}(\theta)
&=\frac{\pi_\theta(y_{i,t}\mid x,\tau_{i,<t})}
{\pi_{\theta_{\mathrm{old}}}(y_{i,t}\mid x,\tau_{i,<t})},
\label{eq:replay-ratio}\\
\mathcal L_{\mathrm{RL}}(\theta)
&=-\frac{1}{\sum_i|\mathcal T_i|}
\sum_i\sum_{t\in\mathcal T_i}
\min\!\left\{\rho_{i,t}(\theta)A_i,
\operatorname{clip}(\rho_{i,t}(\theta))A_i\right\}.
\label{eq:loss}
\end{align}
Here, $\operatorname{clip}(\rho)$ restricts $\rho$ to $[1-\epsilon,1+\epsilon]$, where $0<\epsilon<1$. Thus, ratios sufficiently close to $1$ are unchanged: $\operatorname{clip}(\rho)=\rho$ whenever $|\rho-1|<\epsilon$.
In the following theorem, we connect this objective to our hybrid policy optimization over complete recorded hybrid trajectories.

\vspace{5pt}

\begin{theorem}[Gradient Equivalence in Hybrid Policy Optimization]
\label{thm:rl}
Based on the rollout replay (Section \ref{sec:hybridrl}), record each trajectory's latent inputs and tool observations. 
Under Theorem~\ref{thm:kv} and the replay conditions in Algorithm \ref{alg:hybrid_rl}, the accumulated text-token gradient is equivalent to the advantage-weighted likelihood gradient for the complete hybrid trajectories:
\begin{equation}
\label{eq:credit}
-\nabla_\theta\mathcal L_{\mathrm{RL}}(\theta_{\mathrm{old}})
=\left.
\frac{1}{\sum_i|\mathcal T_i|}\sum_i A_i
\nabla_\theta\log\pi_\theta(\tau_i\mid x)
\right|_{\theta=\theta_{\mathrm{old}}}.
\end{equation}
Therefore, the text-token update assigns advantage-weighted policy credit to optimize the complete recorded hybrid trajectories.
\end{theorem}

\vspace{5pt}
\begin{remark}
At $\theta=\theta_{\mathrm{old}}$, HybridRL produces the same policy update as optimizing the advantage-weighted recorded-trajectory objective.
By preserving the context that produced each reward, exact rollout replay enables the backbone to learn from complete recorded hybrid trajectories.
\end{remark}

\begin{proof}
For trajectory $\tau_i$, we omit the sample index on $E$ and $e$ for notation simplicity. At step $t$ of the $i$-th hybrid trajectory $\tau_i$, the subsequent input can be denoted by
\[
e_{t}=
\begin{cases}
\operatorname{Emb}(y_{i,t}),&\text{generated text token},\\
z_{i,t},&\text{recorded latent input},\\
\operatorname{Emb}(o_{i,t}),&\text{tool observation}.
\end{cases}
\]
During replay, $z_{i,t}$ denotes the latent input recorded under $\theta_{\mathrm{old}}$ and is fixed when differentiating with respect to the current parameters $\theta$, leading to $\frac{\partial z_{i,t}}{\partial\theta}=0$.
Here, although the latent inputs are fixed, $h_{t-1}=f_\theta(E_{\le t-1})$ remains parameter-dependent and retains the chain-rule contribution
\(
\left(\frac{\partial h_{t-1}}{\partial\theta}\right)^{\!\top}
\nabla_{h_{t-1}}
\log\pi_\theta(y_{i,t}\mid x,\tau_{i,<t}),
\)
allowing trajectory-level rewards to optimize the backbone's processing of the complete recorded hybrid context, while keeping the latent inputs fixed.

Then, by causality and Theorem~\ref{thm:kv}, full-sequence replay scores each sampled text token at position $t\in\mathcal T_i$ using the hidden state $h_{t-1}=f_\theta(E_{\le t-1})$ computed from the recorded hybrid prefix.
With this prefix, taking the logarithm and differentiating
\eqref{eq:rl_factorization_app} gives
\begin{equation}
\label{eq:score}
\nabla_\theta\log\pi_\theta(\tau_i\mid x)
=\sum_{t\in\mathcal T_i}
\nabla_\theta\log\pi_\theta(y_{i,t}\mid x,\tau_{i,<t}),
\end{equation}
which expresses the complete recorded-trajectory score as a sum of text-token scores. To connect this equality to the implemented loss, we next show that each token-ratio gradient equals its text-token score at the rollout policy $\theta_{\mathrm{old}}$.

The rollout probability in the denominator of~\eqref{eq:replay-ratio} is fixed during differentiation. Therefore,
\begin{align*}
\nabla_\theta\rho_{i,t}(\theta)
&=
\frac{\nabla_\theta\pi_\theta(y_{i,t}\mid x,\tau_{i,<t})}
     {\pi_{\theta_{\mathrm{old}}}(y_{i,t}\mid x,\tau_{i,<t})}\\
&=
\frac{\pi_\theta(y_{i,t}\mid x,\tau_{i,<t})}
     {\pi_{\theta_{\mathrm{old}}}(y_{i,t}\mid x,\tau_{i,<t})}
\frac{\nabla_\theta\pi_\theta(y_{i,t}\mid x,\tau_{i,<t})}
     {\pi_\theta(y_{i,t}\mid x,\tau_{i,<t})}\\
&=
\rho_{i,t}(\theta)\,
\nabla_\theta\log\pi_\theta(y_{i,t}\mid x,\tau_{i,<t}),
\end{align*}
where the last equality uses
$\nabla_\theta\log p=(\nabla_\theta p)/p$.
At $\theta=\theta_{\mathrm{old}}$, exact replay reproduces the rollout probability, giving $\rho_{i,t}(\theta_{\mathrm{old}})=1$.
Therefore, evaluating the derivative yields
\[
\left.\nabla_\theta\rho_{i,t}(\theta)\right|_{\theta=\theta_{\mathrm{old}}}
=
\left.\nabla_\theta\log\pi_\theta(y_{i,t}\mid x,\tau_{i,<t})
\right|_{\theta=\theta_{\mathrm{old}}}.
\]

Since clipping leaves ratios near $1$ unchanged, each summand in~\eqref{eq:loss} equals $\rho_{i,t}(\theta)A_i$ locally, leading to
\begin{align*}
-\left.\nabla_\theta\mathcal L_{\mathrm{RL}}(\theta)
\right|_{\theta=\theta_{\mathrm{old}}}
&=
\left.
\frac{1}{\sum_i|\mathcal T_i|}
\sum_i A_i\sum_{t\in\mathcal T_i}
\nabla_\theta\rho_{i,t}(\theta)
\right|_{\theta=\theta_{\mathrm{old}}}\\
&=
\left.
\frac{1}{\sum_i|\mathcal T_i|}
\sum_i A_i\sum_{t\in\mathcal T_i}
\nabla_\theta\log\pi_\theta(y_{i,t}\mid x,\tau_{i,<t})
\right|_{\theta=\theta_{\mathrm{old}}}\\
&=
\left.
\frac{1}{\sum_i|\mathcal T_i|}
\sum_i A_i
\nabla_\theta\log\pi_\theta(\tau_i\mid x)
\right|_{\theta=\theta_{\mathrm{old}}},
\end{align*}
where the second equality uses the token-ratio derivative at $\theta_{\mathrm{old}}$, and the final equality uses \eqref{eq:score}.

\end{proof}

\newpage
\section{Extended Related Work}
\label{app:related_work}

\textbf{Adaptive Reasoning on Text and Latent Space.}
Beyond fixed-length chains of thought, recent work lets LLMs adapt how long and in which medium they reason \citep{bigverdi2024perception,lv2026beyond, yang2026demystifying, ye2026latentchem}, for example by learning when to stop latent reasoning or deciding whether explicit deliberation is needed \citep{ning2025learning}. 
Several methods interleave latent reasoning with text decoding \citep{yue2025hybrid}, using confidence or entropy signals \citep{xu2026thinkrouter} to shift into continuous representations at uncertain steps and explore alternative reasoning paths \citep{fu2026selar, shi2026swireasoning}.
However, existing methods largely operate within one reasoning pass or switch modes only at single-turn granularity, and their latent states remain largely opaque to inspect directly. Orthogonal to this line, \ours uses latent mode for efficient deeper reasoning, carries latent context across multiple turns through continuous working memory, keeps latent thoughts directly readable, and preserves the agentic tool-use abilities for long-horizon tasks.

\textbf{Hybrid Communication among Agents.}
A recent line of work in multi-agent systems (MAS) parallel to \ours studies \textit{hybrid-protocol communication}, in which each inter-agent message combines two media. A latent channel carries dense intermediate signals such as hidden states or KV entries, while a text channel carries concise commitments that remain inspectable.
Early instances augment each natural-language message with the sender's internal state trajectory \citep{tang-etal-2025-augmenting}, and later work formalizes the dual-channel design by routing elaborate reasoning through latent vectors while keeping answers and decisions in text \citep{mou2026hylat,dai2026token}.
More recent works take the same view from the deployment side, pairing every latent payload with a visible text commitment wherever oversight is required \citep{liu2026tokensunifiedframeworklatent} and evaluating the security of latent transfer alongside visible text \citep{brito2026latentagentsliekvcache}. Complementary to these inter-agent protocols in the MAS setting, \ours decides the medium at the level of each generation step within a single agent, switching between latent reasoning and text-based tool interaction over multi-turn trajectories.

\textbf{KV-Cache Reuse.}
Reusing key--value (KV) entries across generation steps is standard practice for efficient serving \citep{kwon2023efficient,zheng2024sglang}, and recent work extends it across requests, caching shared prefixes and prompt modules \citep{gim2024prompt,ye2024chunkattention,juravsky2024hydragen} and composing independently encoded caches so that repeated context need not be prefilled again \citep{yang2025kvlink,ma2025blockattention,hu2025epic}. These methods treat the KV cache as an optimization over a context that could in principle be re-encoded from text. \ours likewise builds on the backbone's own KV entries, using them as a unified working memory that stores and propagates both latent and text information across multi-turn hybrid generation.

\newpage
\section{Hybrid Data Curation}
\label{app:annotation}

\begin{figure}[!h]
    \centering
    \includegraphics[width=\linewidth]{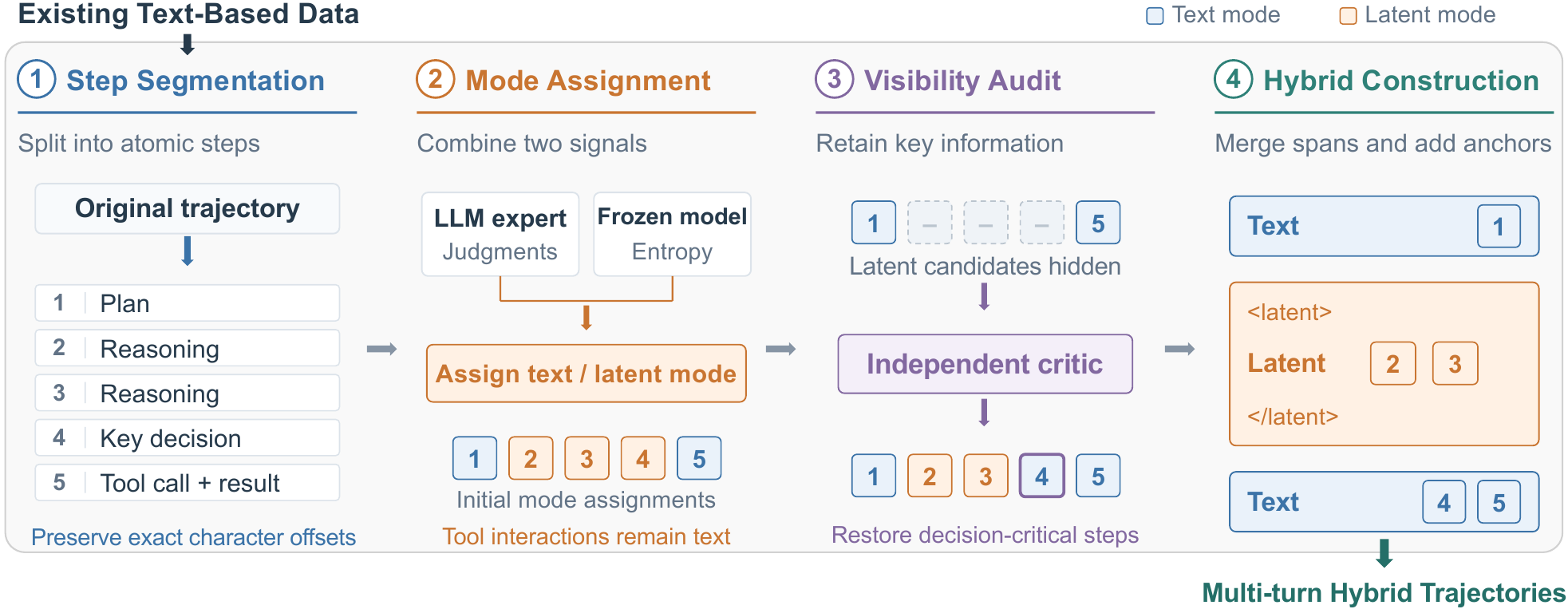}
    \caption{\textbf{Hybrid Data Curation Pipeline.} We convert text-based demonstrations into multi-turn hybrid trajectories using complementary signals from LLM judgments and model entropy to assign reasoning modes, followed by a visibility audit that preserves decision-critical steps in text mode.}
    \label{fig:data_label}
    \vspace{-5pt}
\end{figure}

We curate the hybrid training data by converting existing multi-turn agentic demonstrations \citep{yu2025dapo, dong2025agentic} into 54k hybrid trajectories for text and latent supervision.
Specifically, an automatic annotation pipeline is designed to assign each source CoT step to either text or latent mode, based on complementary signals from external LLM judges and the target model's entropy-based confidence, with no human labeling needed.
Accordingly, high-confidence reasoning steps that the model can handle reliably are assigned as latent targets, while audit-critical reasoning and tool-interaction steps are retained as text targets. We describe the detailed curation pipeline below.

\stepcircle{1} \textbf{Step Segmentation.}
We first split each source trajectory into atomic reasoning steps using a deterministic procedure, recording their exact character offsets. These offsets allow subsequent annotations to identify and group steps while preserving the original content and ordering.

\stepcircle[steporange]{2} \textbf{Mode Assignment.}
We then leverage external LLM experts (GPT-6.1-Sol/GPT-5.6-Luna) to assign each reasoning step to text or latent mode, with tool interactions and key steps needed for external inspection retained in text mode (See Prompts in Appendix \ref{app:prompt_template}). Meanwhile, we complement the external selection with the target model's prediction entropy. We teacher-force the original textual trajectory through a frozen base model and compute the mean token-level entropy for each step $s$:
\begin{equation}
H_s
=
-\frac{1}{|\mathcal{I}_s|}
\sum_{t\in\mathcal{I}_s}
\sum_{v\in\mathcal{V}}
f_{\theta}(v\mid x,\tau_{<t})
\log f_{\theta}(v\mid x,\tau_{<t}),
\label{eq:curation_entropy}
\end{equation}
where $f_\theta(v\mid x,\tau_{<t})$ denotes the frozen base model's next-token probability after the vocabulary projection and softmax. Here, $x$ is the input question, $\tau_{<t}$ contains preceding reasoning and tool observations, $\mathcal{I}_s$ indexes the tokens in step $s$, and $\mathcal{V}$ is the vocabulary. Low-entropy steps become latent candidates, while high-entropy steps remain in text mode. Initial mode assignments combining the two complementary signals are then refined through the visibility audit below.

\stepcircle[steppurple]{3} \textbf{Visibility Audit.}
An independent LLM critic next examines the visible version of each annotated trajectory, with candidate latent spans hidden. The critic assesses whether the remaining text and tool interactions provide sufficient information to follow the agent's decisions. When decision-critical information is missing, the corresponding source steps are restored to text mode. This audit helps preserve the key information needed for external debugging and inspection.

\stepcircle[stepteal]{4} \textbf{Hybrid Trajectory Construction.}
We merge adjacent latent-assigned steps into continuous spans and programmatically insert the special anchors \texttt{<latent>} and \texttt{</latent>}. Tool calls and observations remain in text form. Finally, we verify that removing the inserted anchors recovers the original trajectory byte for byte. This consistency check ensures that the pipeline annotates reasoning modes without rewriting or altering the source demonstrations.

\section{Latent Interpretability Evaluation (Figures \ref{fig:probe_method} and \ref{fig:HybridSFT_probe_all})}
\label{app:latent_probe}

In this section, we describe the evaluation setup and metrics used to measure the readability and context specificity of latent readouts in Figures~\ref{fig:probe_method} and~\ref{fig:HybridSFT_probe_all}.

\textbf{Probe Setup.}
We evaluate \ours across 200 questions spanning downstream tasks, including AIME25/26 \citep{aime25,aime26}, MATH-500 \citep{math500}, MuSiQue \citep{musique}, GAIA \citep{mialon2024gaia}, and DSQA \citep{dsqa}. Figure~\ref{fig:probe_method} compares the final HybridRL checkpoints of \ours at 4B, 8B, and 14B scales. Figure~\ref{fig:HybridSFT_probe_all} compares format activation, \link activation, self-conditioned HybridSFT, and Coconut on the shared Qwen3-14B-Base backbone. Generation uses temperature 0.6, top-p 0.95, top-k 20, and a fixed random seed. Each unit inside the plots is a \emph{latent span}: a maximal sequence of consecutive latent steps. Each plotted point represents one span, and rings mark the means of per-span scores for each model or variant. The annotation $n=200$ refers to total number of questions.

\textbf{Readability.}
For a latent span $b$ containing $L_b$ latent inputs $z_{b,t}$, we measure \textit{readability score} through their average nearest-token alignment, following \eqref{eq:latent-exit}:
\begin{equation}
\mathrm{Read}(b)
=
\frac{1}{L_b}\sum_{t=1}^{L_b}
\max_{v\in\mathcal{V}}
\cos\!\left(z_{b,t},\mathrm{Emb}(v)\right),
\label{eq:readability}
\end{equation}
where $\mathcal{V}$ is the vocabulary and $\mathrm{Emb}(v)$ is the input embedding of token $v$. A score of $1$ indicates that every latent input shares the direction of a vocabulary embedding, corresponding to the ``exact token'' reference. This alignment serves as a geometric proxy for readability. To calibrate the noise reference, we apply the same readout to 2{,}000 Gaussian random vectors and use their mean nearest-token similarity as the threshold $\tau_{\mathrm{noise}}$. Spans above this threshold exhibit stronger average token alignment than the Gaussian reference.

\textbf{Context Specificity.}
Context specificity measures whether a span's pseudo-text is more relevant to its originating question than to other questions from the same task. Let $P_b$ denote the pseudo-text obtained by nearest-token readout of span $b$, as in \eqref{eq:latent-exit}. Let $W(s)$ extract the content words from string $s$: lower-cased alphanumeric tokens longer than two characters, with stop words removed. We use Jaccard overlap,
$J(A,B)=|A\cap B|/|A\cup B|$.
For a task-specific question set $\{q_1,\ldots,q_M\}$ with $M\geq2$, let $i(b)$ identify the question that produced span $b$. We define the \textit{context specificity score} as:
\begin{equation}
\mathrm{Spec}(b)
=
\frac{J\!\left(W(P_b),W(q_{i(b)})\right)}
{\dfrac{1}{M-1}
\displaystyle\sum_{\substack{j=1\\j\neq i(b)}}^{M}
J\!\left(W(P_b),W(q_j)\right)},
\label{eq:specificity}
\end{equation}
when the denominator is positive. A score of $1\times$ indicates equal overlap, while higher scores above $1\times$ indicate greater lexical relevance to the originating question. We exclude readouts without content words and plot positive scores on a logarithmic axis.

\textbf{Meaningful Region.}
The shaded region contains spans whose readability exceeds $\tau_{\mathrm{noise}}$ and whose context specificity exceeds $1$. The two axes are complementary: readability measures alignment with token embeddings, while context specificity measures relevance to the originating question. Together, they provide evidence of readable, question-specific latent readouts that support external inspection.

\section{Additional Experiments}
\label{app:additional_exps}

\subsection{Experiments on 4B-Scale \ours}
\label{app:4b_results}

\begin{table*}[!h]
\centering
\caption{
\textbf{4B-Scale \ours on Math and Knowledge-Intensive Search.} We evaluate performance across both SFT and RL stages. AIME25/26 are reported on avg@16 accuracy.
}
\label{tab:4B_hybrid}
\setlength{\tabcolsep}{4pt}
\resizebox{0.7\textwidth}{!}{%
\begin{tabular}{lcccc}
\toprule
\textbf{Method}
& \textbf{AIME25} & \textbf{AIME26} & \textbf{MATH500}
& \textbf{GPQA-D} \\

\midrule
\noalign{%
  \begingroup
  \color{gray!10}%
  \vskip-\belowrulesep
  \hrule height\belowrulesep
  \endgroup
}
\rowcolor{gray!10}
\multicolumn{5}{l}{%
  \raisebox{0.5pt}[\height][\depth]{\textbf{Qwen3-4B-Base}}%
} \\
Coconut & 10.6 & 13.3 & 40.6 & 17.2 \\
SoftCoT & 14.8 & 18.5 & 55.2 & 26.3 \\
Standard Text SFT & 19.0 & 22.7 & 77.4 & 31.8 \\
\textbf{\ours (SFT)} & 22.5 & 23.3 & 76.8 & 30.8 \\
\midrule
GRPO & 28.1 & 35.6 & 80.6 & 36.9 \\
ARPO & \underline{31.3} & \underline{36.3} & \underline{81.8} & \underline{38.4} \\
\textbf{\ours (RL)} & \gaincell{35.2} & \gaincell{37.5} & \gaincell{82.2} & \gaincell{40.9} \\[-0.4ex]
\bottomrule
\end{tabular}}
\end{table*}

To examine the effectiveness of hybrid reasoning at smaller model scales, we evaluate \ours and corresponding text-based and latent baselines on Qwen3-4B-Base.
As shown in Table~\ref{tab:4B_hybrid}, HybridSFT outperforms both latent baselines across all four benchmarks and improves over standard text SFT on AIME25 and AIME26.
HybridRL further improves performance across all four benchmarks, achieving the best results and exceeding ARPO by an average of 2.0 percentage points.
Together with 8B and 14B comparisons in Table \ref{tab:main_agentic} and \ref{tab:main_others}, these results show that the benefits of hybrid reasoning can extend to different model scales of \ours across downstream reasoning tasks.

\subsection{Cost Analyses on \ours}
\label{app:cost_analyses}

\begin{table}[!h]
\centering
\caption{
\textbf{Cost Analyses on \ours.}
We compare \ours with text agents under the same backbones, reporting compute per question and estimated monetary cost per 1M questions.
Each text or latent generation step is charged as one full backbone forward pass.
Monetary costs assume \$2.50 per H100-hour and 10\% utilization of peak BF16 throughput (approximately \$0.007 per PFLOP).
}
\label{tab:inference_cost}
\small
\resizebox{0.8\textwidth}{!}{%
\begin{tabular}{ccccc}
\toprule
\multirow{2}{*}{\textbf{Model Backbone}} & \textbf{Text Agent} & \textbf{HybridAgent} & \textbf{Cost / 1M questions} & \multirow{2}{*}{\textbf{Cheaper by}} \\
 & (TFLOP) & (TFLOP) & (Text Agent $\rightarrow$ HybridAgent) & \\
\midrule
Qwen3-4B-Base  & 51.2 & 13.9 & \$359 $\rightarrow$ \$98  & 3.7$\times$ \\
Qwen3-8B-Base  & 62.5 & 17.2 & \$439 $\rightarrow$ \$121 & 3.6$\times$ \\
Qwen3-14B-Base & 66.9 & 25.1 & \$470 $\rightarrow$ \$176 & 2.7$\times$ \\
\bottomrule
\end{tabular}
}
\end{table}

\textbf{Cost Analyses on \ours.} To examine whether the efficiency of hybrid generation translates into lower computational cost, we compare \ours{} with matched text agents across 4B, 8B, and 14B backbones. We estimate decoding compute by charging every text or latent step a full backbone forward pass, together with the cost of vocabulary projection or the HybridLink adapter, respectively. Step counts are averaged over matched evaluations with the same questions, generation budgets, and tool-call limits. We further estimate monetary costs assuming \$2.50 per H100-hour and 10\% utilization of peak BF16 throughput. These estimates cover generation compute and exclude prompt prefill, context-dependent attention, and external tool execution.
As shown in Table \ref{tab:inference_cost}, \ours{} consistently reduces estimated decoding compute across all three model scales, from 51.2--66.9 TFLOP to 13.9--25.1 TFLOP per question, corresponding to reductions of 62.5\%--72.9\%. Under the stated pricing and utilization assumptions, the estimated cost of answering one million questions decreases from \$359 to \$98 at 4B, from \$439 to \$121 at 8B, and from \$470 to \$176 at 14B, yielding cost reduction factors of $3.7\times$, $3.6\times$, and $2.7\times$, respectively. Together with the shorter generation trajectories in Figure \ref{fig:token_eff}, these results demonstrate that hybrid reasoning reduces overall decoding compute while explicitly accounting for the backbone computation incurred by latent steps.

\section{Experiment Setups}
\label{app:exp_setups_overall}

\subsection{Benchmark and Evaluation Details}
\label{app:benchmark}

We evaluate \ours on 11 benchmarks spanning mathematical, knowledge-intensive, and agentic reasoning.
Together, these benchmarks assess complementary capabilities in multi-step problem solving, specialized knowledge application, evidence integration, and tool-assisted information seeking. We follow the setup of \citet{dong2025agentic} and describe each benchmark in detail below.

\textbf{(i) Mathematical Reasoning Tasks}

\textit{AIME 25} \citep{aime25}:
AIME 25 comprises 30 problems from the 2025 American Invitational Mathematics Examination, covering algebra, geometry, number theory, and combinatorics.
Its integer-answer problems evaluate competition-level mathematical reasoning, requiring models to combine mathematical insights with precise multi-step calculations.

\textit{AIME 26} \citep{aime26}:
AIME 26 contains 30 problems from the 2026 American Invitational Mathematics Examination, following the same competition format as AIME 25.
It provides a complementary assessment of mathematical problem solving across a separate annual problem set, testing sustained reasoning and accurate numerical derivation.

\textit{MATH500} \citep{math500}:
MATH500 is a 500-problem subset of the MATH test set, spanning 7 mathematical subjects and 5 difficulty levels.
It evaluates math reasoning across diverse problem types, including symbolic manipulation, geometric deduction, and numerical computation.

\textbf{(ii) Knowledge-Intensive Reasoning Tasks}

\textit{GPQA-Diamond} \citep{gpqa}:
GPQA-Diamond is a curated subset of 198 expert-written multiple-choice questions in biology, physics, and chemistry.
Selected through expert and non-expert validation, these questions assess advanced scientific knowledge and reasoning that remain challenging for non-experts even with web access.

\textit{MedQA} \citep{medqa}:
MedQA is a multiple-choice question-answering benchmark derived from professional medical licensing examinations. The dataset contains 300 questions in total.
Its questions assess medical knowledge and clinical reasoning, requiring models to interpret patient information and identify appropriate diagnoses, investigations, or treatments.

\textit{MuSiQue} \citep{musique}:
MuSiQue is a multi-hop question-answering benchmark constructed by composing single-hop questions into connected reasoning chains of two to four hops.
It evaluates the ability to combine evidence across multiple passages, with question construction designed to reduce shortcuts that bypass intermediate reasoning steps. Following \citet{dong2025agentic}, we use 200 test-set questions for downstream evaluation across all methods.

\textbf{(iii) Agentic Reasoning Tasks}

\textit{General AI Assistants (GAIA)} \citep{mialon2024gaia}:
GAIA evaluates general-purpose assistants on real-world questions involving reasoning, web browsing, information processing, and external tool use.
Its three difficulty levels reflect increasingly demanding combinations of reasoning steps and tool interactions, assessing the ability to coordinate these capabilities toward a verifiable answer.

\textit{Humanity's Last Exam (HLE)} \citep{hle}:
HLE comprises expert-authored questions across a broad range of academic disciplines, including mathematics, natural sciences, and the humanities.
Its multiple-choice and short-answer problems assess advanced subject knowledge and reasoning on questions designed to resist quick answers through direct web retrieval.

\textit{WebWalker} \citep{webwalker}:
WebWalkerQA, referred to as WebWalker in our evaluations, is an agentic search benchmark requiring exploration of website subpages and integration of information from one or multiple sources.
Its easy, medium, and hard categories reflect increasing traversal depth, testing the ability to locate and synthesize evidence beyond surface-level search results.

\textit{SealQA} \citep{pham2026sealqa}:
SealQA evaluates search-augmented question answering on factual questions for which retrieved information may be noisy, conflicting, or unhelpful.
We evaluate on Seal-0 and Seal-Hard, which test factual accuracy and reasoning under challenging retrieval conditions.
Seal-0 targets questions selected for near-zero accuracy among the non-reasoning models used during benchmark construction; Seal-Hard includes a broader collection of difficult questions.

\textit{DeepSearchQA (DSQA)} \citep{dsqa}:
DeepSearchQA evaluates multi-step web research through tasks that require assembling complete answer sets from information distributed across sources.
It tests search planning, evidence aggregation, entity disambiguation, and the ability to determine when sufficient information has been collected.
We report both full-correct rate and F1 score to assess complete task success and the balance between answer precision and coverage.

\subsection{Baseline Details}
\label{app:baseline}

To evaluate the effectiveness and efficiency of \ours, we compare with 4 categories of methods, including (i) Text-based Agentic Methods, (ii) Latent Reasoning Methods, (iii) Frontier LLM Agents, and (iv) Agentic Frameworks. We describe the implementation details of each baseline below.

\textbf{(i) Text-based Agentic Methods}

\textit{Standard Text SFT}: We apply standard SFT using the same training data and backbone models as \ours. During evaluation, the models can invoke tools and incorporate the returned observations into subsequent reasoning. We use this baseline mostly to compare with the HybridSFT.

\textit{Agentic RL (GRPO/ARPO)}: We apply commonly adopted agentic RL methods after standard text SFT to further train the text agent baselines. GRPO \citep{shao2024deepseekmath} optimizes the policy using group-normalized rewards from sampled responses without requiring a separate critic. ARPO \citep{dong2025agentic}: ARPO combines entropy-guided adaptive rollouts with step-level advantage attribution to improve multi-turn tool use. We train these methods using the same RL training data and backbone models as \ours, primarily to evaluate the effectiveness of HybridRL.

\textbf{(ii) Latent Reasoning Methods}

\textit{Coconut} \citep{hao2024training}: Coconut performs continuous latent reasoning by directly feeding the model's last hidden state back as the next input embedding. A staged training curriculum progressively replaces textual reasoning steps with latent thoughts, teaching the model to reason in continuous space before generating a textual answer.

\textit{SoftCoT} \citep{xu-etal-2025-softcot}: SoftCoT uses a lightweight frozen assistant model to generate question-dependent soft thought representations, which a trainable projection maps into the target LLM's representation space. These representations guide subsequent generation, enabling latent reasoning while keeping the assistant and target LLM frozen.

\textit{SoftThinking} \citep{zhang2025soft}: SoftThinking is a context-based method that represents intermediate reasoning as probability-weighted mixtures of vocabulary embeddings. Feeding these hidden continuous representations back into the model preserves alternatives from the next-token distribution.

We implement all latent reasoning baselines using the same backbone models and, where applicable, training settings and budgets as \ours to ensure a fair comparison.

\textbf{(iii) Frontier LLM Agents}

To evaluate \ours in a broader context, we compare it with current leading closed-source and open-source LLM agents, including GPT-5.6-Luna \citep{openai_gpt56_luna}, Kimi-K3 \citep{kimik3}, GLM-5.3-Flash \citep{glm5}, DeepSeek-V4.1-Flash \citep{deepseekv41flash}, and QwQ-32B \citep{qwq32b}, to assess both performance and inference efficiency of \ours.

\textbf{(iv) Agentic Frameworks}

We further compare \ours with established agentic frameworks that integrate multi-step reasoning with external tool calling and environmental interactions. We apply Qwen3-14B as the backbone for all these frameworks, with the same tool access and usage budgets as \ours.

\textit{ReAct} \citep{yao2023react}: ReAct interleaves natural-language reasoning with actions and observations from an external environment. This interaction loop allows an agent to gather information, revise its reasoning, and determine subsequent actions based on tool feedback.

\textit{Search-o1} \citep{search-o1}: Search-o1 augments large reasoning models with adaptive search to address knowledge gaps encountered during reasoning. Its Reason-in-Documents module analyzes retrieved documents and extracts relevant evidence before incorporating it into ongoing reasoning.

\textit{WebThinker} \citep{li2025webthinker}: WebThinker equips LLMs with a Deep Web Explorer that supports iterative search, webpage navigation, and information extraction. Its Think-Search-and-Draft strategy interleaves reasoning, evidence collection, and report generation to support complex research tasks.

Overall, these baseline methods provide complementary comparisons across training stages, text-based and latent-space reasoning, and tool-integrated agentic settings.

\subsection{Implementation Details}
\label{app:imple_setup}

\textbf{Models and Training Details.} We instantiate \ours with Qwen3-Base backbones \citep{yang2025qwen3} at three scales of 4B, 8B, and 14B.
For \link, we use the same two-layer design across all three backbones. The HybridSFT training data is constructed by converting 54K text-only sources \citep{dong2025tool,dong2025agentic} into hybrid trajectories. We apply format and \link activation for one epoch each, followed by one additional epoch of self-conditioned training. HybridRL training utilizes 12K deep reasoning examples \citep{yu2025dapo} for math and 1K agentic reasoning and search examples \citep{dong2025agentic} for deep search.
We train \ours with the GRPO objective and set rollout batches of 512 trajectories (\(64\times8\)), followed by one optimization pass over 8 minibatches. 

\textbf{Training Hyperparameters.} Table \ref{tab:hp_sft}, \ref{tab:hp_rl_deepsearch}, and \ref{tab:hp_rl_math} provide detailed parameter setups on HybridSFT and HybridRL. Following \citet{dong2025agentic,yu2025dapo}, we use \ours checkpoints from deep reasoning training for mathematical evaluations and the checkpoint from agentic search training for knowledge-intensive search and agentic reasoning evaluations, as detailed in Appendix~\ref{app:benchmark}.

\begin{table}[!t]
\centering
\caption{Hyperparameter Setups of HybridSFT.}
\label{tab:hp_sft}
\small
\resizebox{0.7\textwidth}{!}{%
\begin{tabular}{ll}
\toprule
\textbf{Hyperparameter} & \textbf{Value} \\
\midrule
Base model & Qwen3-4B/8B/14B-Base \\
Maximum sequence length & 16{,}888 \\
Effective batch size & 16 \\
Backbone learning rate & $7\times10^{-6}$ \\
\link learning rate (Activation / Self-conditioning)
& $2\times10^{-4}$ / $5\times10^{-5}$ \\
Weight decay & 0.0 \\
Warmup ratio & 0.1 \\
Learning rate scheduler & Cosine\\
Latent regression weights ($\lambda_\text{cos}$ / $\lambda_\text{mse}$) & 1.0 / 0.2 \\
Number of training epochs & 3 (1 / 1 / 1)\\
\bottomrule
\end{tabular}
}
\end{table}

\begin{table}[!t]
\centering
\small
\caption{Hyperparameter Setups of HybridRL on Agentic Search.}
\label{tab:hp_rl_deepsearch}
\resizebox{0.7\textwidth}{!}{%
\begin{tabular}{ll}
\toprule
\textbf{Hyperparameter} & \textbf{Value} \\
\midrule
Initialization & 8B/14B HybridSFT checkpoint \\
Unique prompts per rollout batch & 64 \\
Rollouts per prompt (group size) & 8 \\
Minibatches per rollout batch & 8 \\
Optimization epochs per rollout batch & 1 \\
Total rollout batch size & 512 trajectories ($64$ $\times$ $8$) \\
Optimization per rollout batch & 1 pass over 8 minibatches \\
Rollout sampling temperature & 1.0 \\
Rollout sampling top-p & 1.0 \\
Reward $r_\text{correct}$ & token-level F1 $\in[0,1]$\\
Reward $r_\text{format}$ & 0 / -1 \\
Reward $r_\text{tool}$ & 0.1 \\
Latent steps per latent block (min/max) & 16 / 256 \\
Max tool calls per rollout & 10 \\
Weight decay & 0.01 \\
KL penalty coefficient & 0 \\
Learning rate schedule & constant \\
Training steps & 150 \\
Maximum sequence length & 16{,}888 \\
Backbone learning rate & $1\times10^{-6}$ \\
Optimizer & AdamW, $\beta=(0.9,\,0.95)$ \\
Search API Setup (per call) & top-10, 1000 characters each \\
\bottomrule
\end{tabular}
}
\end{table}

\begin{table}[!t]
\centering
\small
\caption{Hyperparameter Setups of HybridRL on Deep Reasoning.}
\label{tab:hp_rl_math}
\resizebox{0.7\textwidth}{!}{%
\begin{tabular}{ll}
\toprule
\textbf{Hyperparameter} & \textbf{Value} \\
\midrule
Initialization & 4B/8B/14B HybridSFT checkpoint \\
Unique prompts per rollout batch & 64 \\
Rollouts per prompt (group size) & 8 \\
Minibatches per rollout batch & 8 \\
Optimization epochs per rollout batch & 1 \\
Total rollout batch size & 512 trajectories ($64 \times 8$) \\
Optimization per rollout batch & 1 pass over 8 minibatches \\
Rollout sampling temperature & 1.0 \\
Rollout sampling top-$p$ & 1.0 \\
Reward $r_{\text{correct}}$ & Math Verify ($0$ / $1$) \\
Reward $r_{\text{format}}$ & $0$ / $-1$ \\
Reward $r_{\text{tool}}$ & $+0.1$ \\
Latent steps per latent span (min/max) & 16 / 256 \\
Max tool calls per rollout & 5 \\
Weight decay & 0.01 \\
KL penalty coefficient & 0 \\
Learning rate schedule & Constant \\
Training steps & 200 \\
Maximum sequence length & 16{,}888 \\
Backbone learning rate & $1\times10^{-6}$ \\
Optimizer & AdamW, $\beta=(0.9,\,0.95)$ \\
\bottomrule
\end{tabular}
}
\end{table}

\textbf{Inference Setups.}
During inference, we set the temperature, top-p, and top-k to 0.6, 0.95, and 20, respectively. The generation budget is adjusted for each task based on its relative difficulty and counts both text tokens and latent steps, with limits of 4,096 for mathematical reasoning and knowledge-intensive search tasks, and 24,000 for challenging agentic reasoning tasks.
For each latent span, we enforce a minimum of 16 and a maximum of 256 steps to ensure sustained and semantically rich latent reasoning.
\ours interacts with external environments primarily through two types of tools: a Python interpreter within sandbox environments and a live search engine (e.g., Tavily). 
We cap tool usage at 10 calls per trajectory and retrieve the top 10 snippets for each search query.
We utilize GPT-5.6-Luna as both the search-page summarizer and LLM judge.
All training and inference are conducted on 32 H100 GPUs. We also provide detailed prompt templates in Appendix \ref{app:prompt_template}.

\textbf{Infrastructure and Configurations.} 
As existing training and inference frameworks such as verl \citep{verl} and vLLM \citep{kwon2023efficient} do not directly support multi-turn hybrid generation, we extend their infrastructure to enable hybrid training and generation in \ours.
A shared generation loop coordinates text mode, latent mode, and tool interaction while maintaining each trajectory's persistent KV cache as continuous working memory across turns.
The agent enters latent mode through the special entry anchor \texttt{<latent>}, feeds \link-generated latent inputs back into the backbone, and returns to text mode when the nearest-token readout matches the exit anchor \texttt{</latent>}, as in \eqref{eq:latent-exit}. Tool observations are tokenized and processed against the existing working memory, appending new KV entries without reconstructing or re-prefilling earlier context.

The engine also supports batched rollouts with per-sequence state tracking, allowing trajectories to occupy different generation states within the same batch. Pending tool calls execute concurrently, with waiting and completed sequences masked from further generation. Per-sequence attention masks and position tracking preserve each trajectory's causal context as generation proceeds. For HybridRL, exact latent-context replay recomputes sampled text-token likelihoods in a single differentiable forward pass while keeping recorded latent inputs unchanged. We check rollout--replay agreement of log-probabilities at unchanged policy parameters. Multi-node data parallelism, sharded optimizer states, mixed precision, gradient checkpointing, and resumable checkpoints support full-parameter backbone optimization. Overall, the same engine supports RL rollouts, training-time validation, and downstream evaluation, ensuring consistent mode transitions and working-memory handling across configurable sampling settings, generation budgets, and tool-call setups.

\section{Prompt Template}
\label{app:prompt_template}

\begin{figure}[H]
    \centering
    \includegraphics[width=0.95\linewidth]{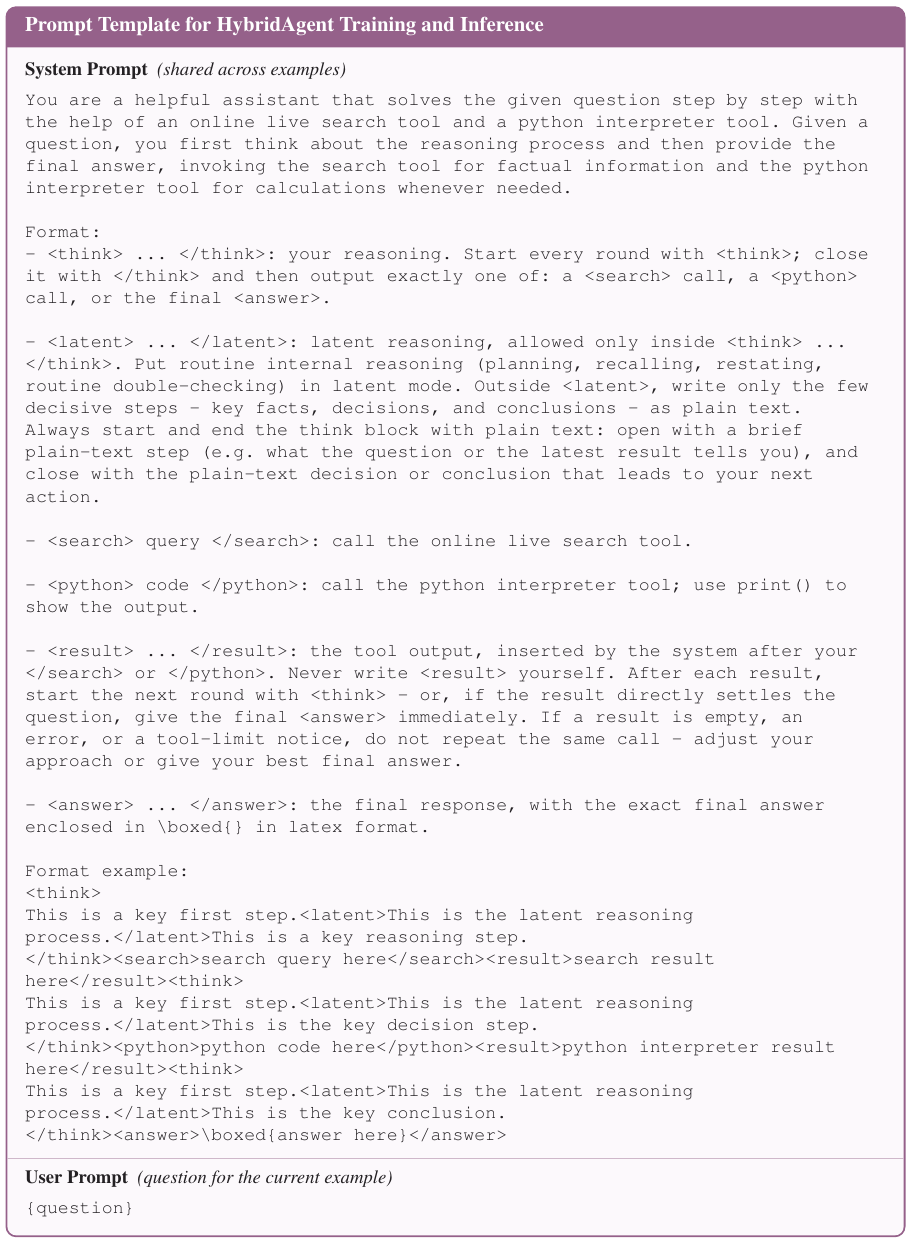}
    \label{fig:template1}
\end{figure}

\begin{figure}[H]
    \centering
    \includegraphics[width=0.95\linewidth]{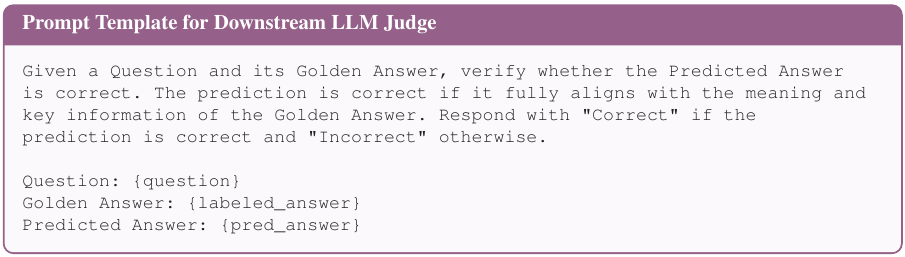}
    \label{fig:template2}
    \vspace{-5pt}
\end{figure}

\begin{figure}[H]
    \centering
    \includegraphics[width=0.95\linewidth]{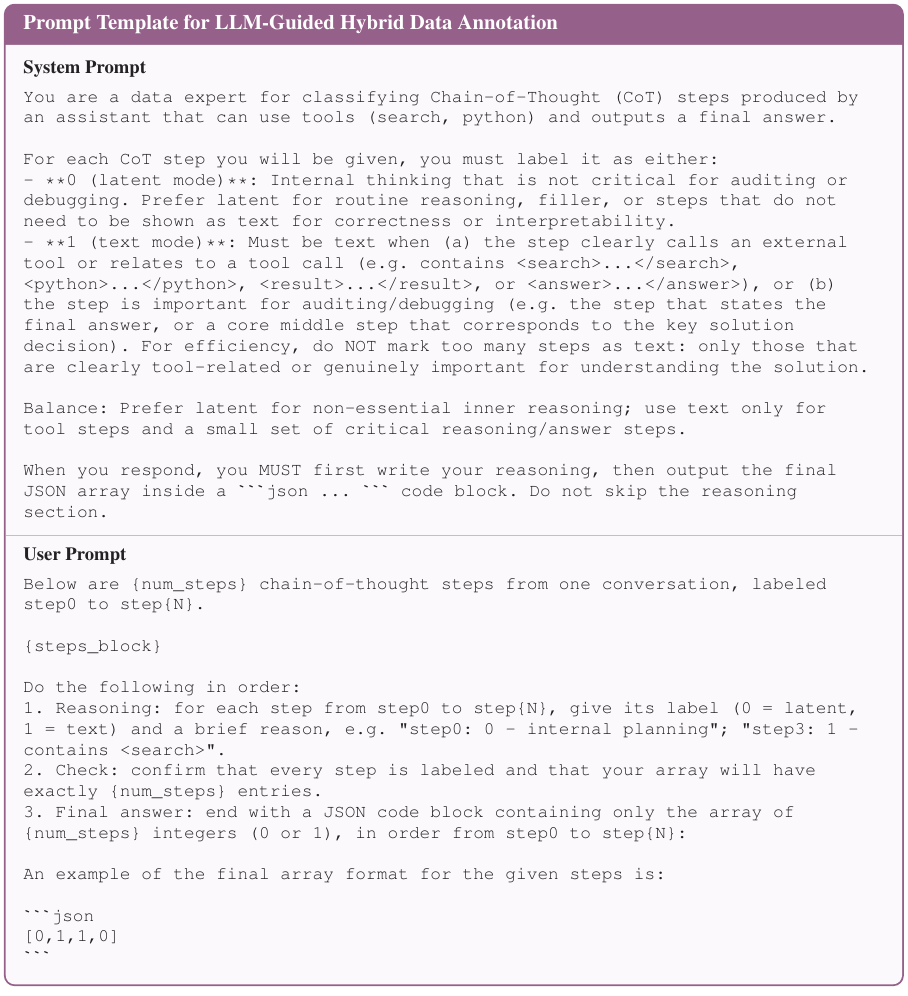}
    \label{fig:template3}
    \vspace{-10pt}
\end{figure}

\section{Case Study of \ours}
\label{app:case_probe}
\vspace{-15pt}

\begin{figure}[H]
    \centering
    \includegraphics[width=0.95\linewidth]{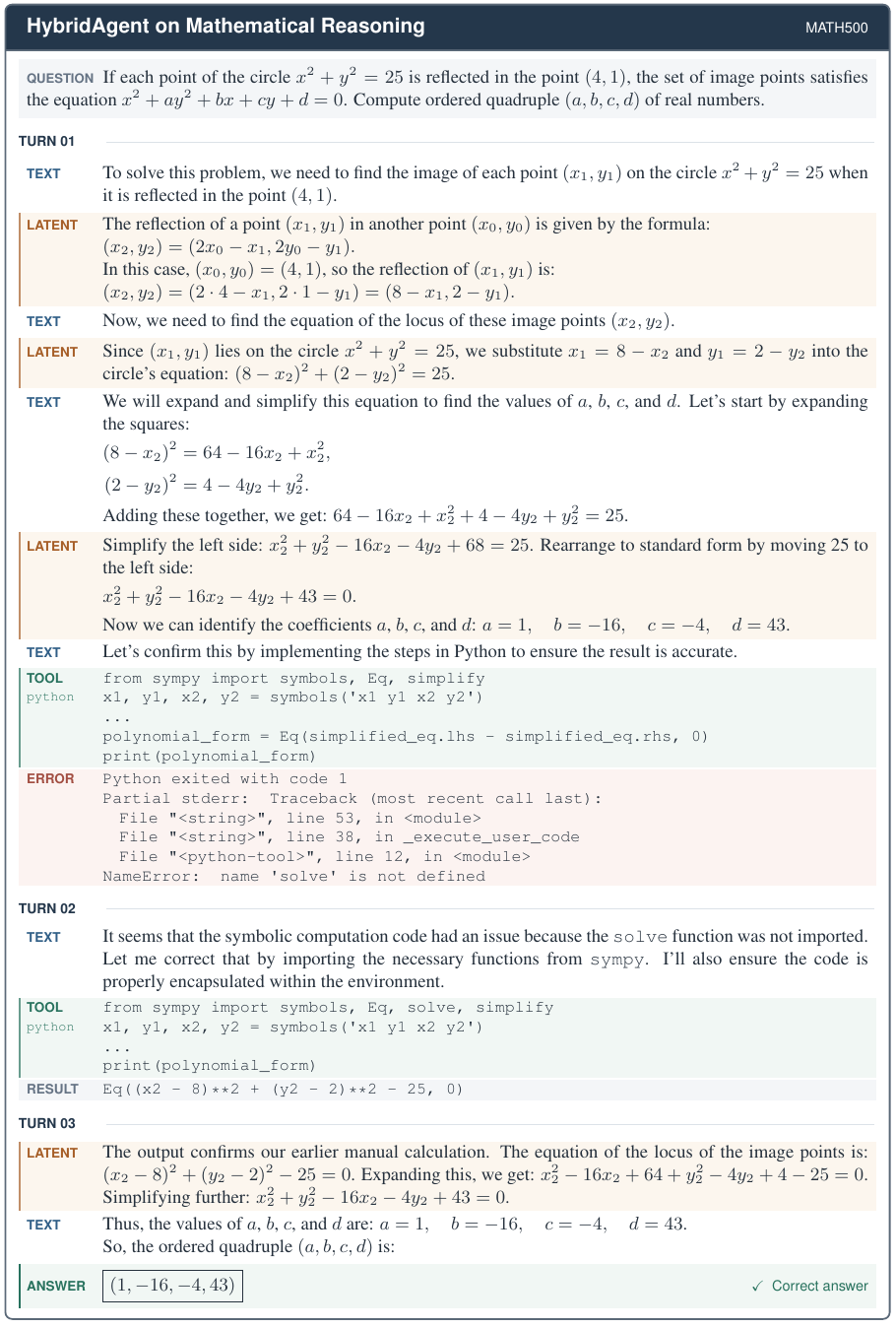}
    \label{fig:case1}
    \vspace{-30pt}
\end{figure}

\begin{figure}[H]
    \centering
    \includegraphics[width=0.95\linewidth]{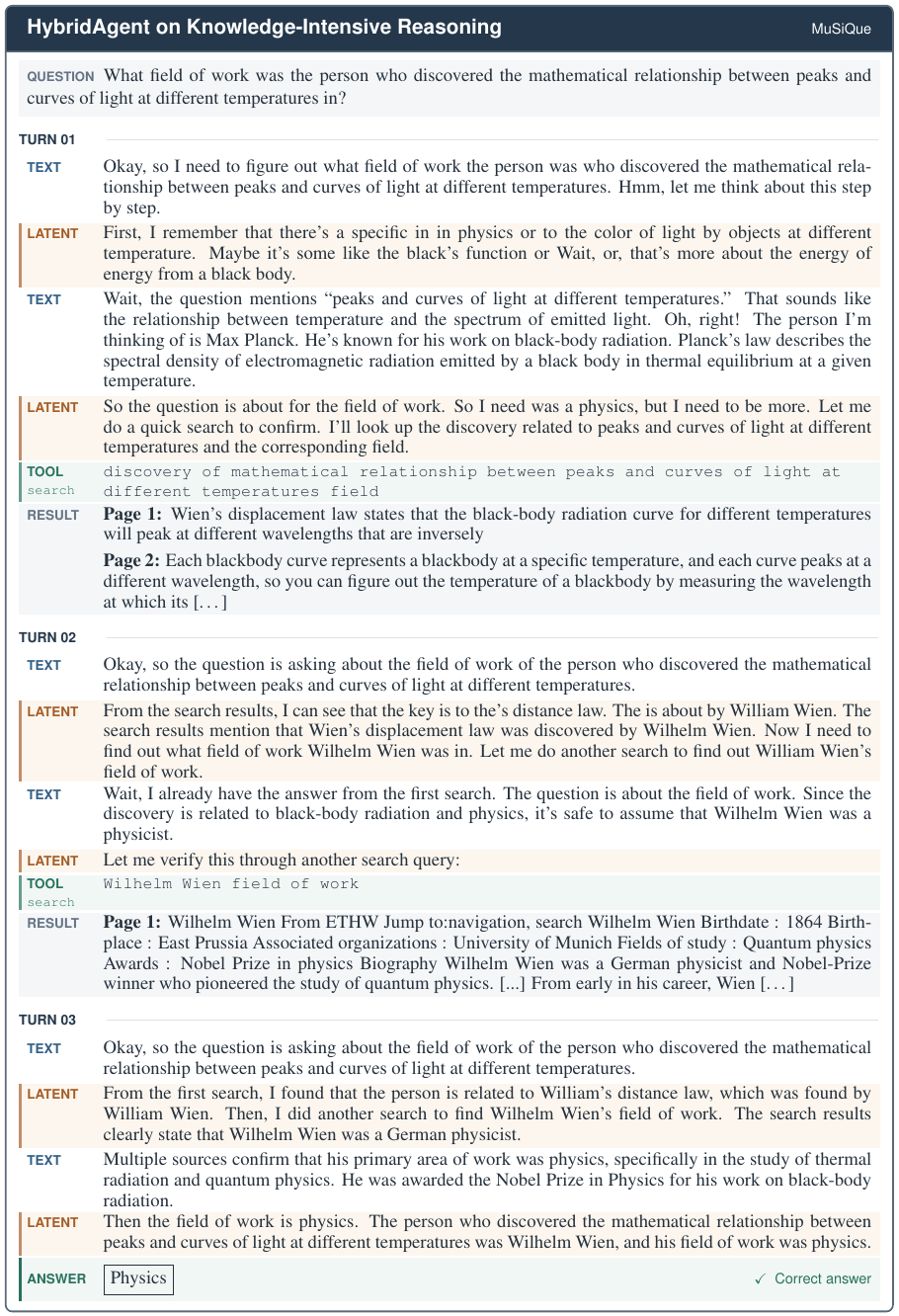}
    \label{fig:case2}
\end{figure}

\begin{figure}[H]
    \centering
    \includegraphics[width=0.95\linewidth]{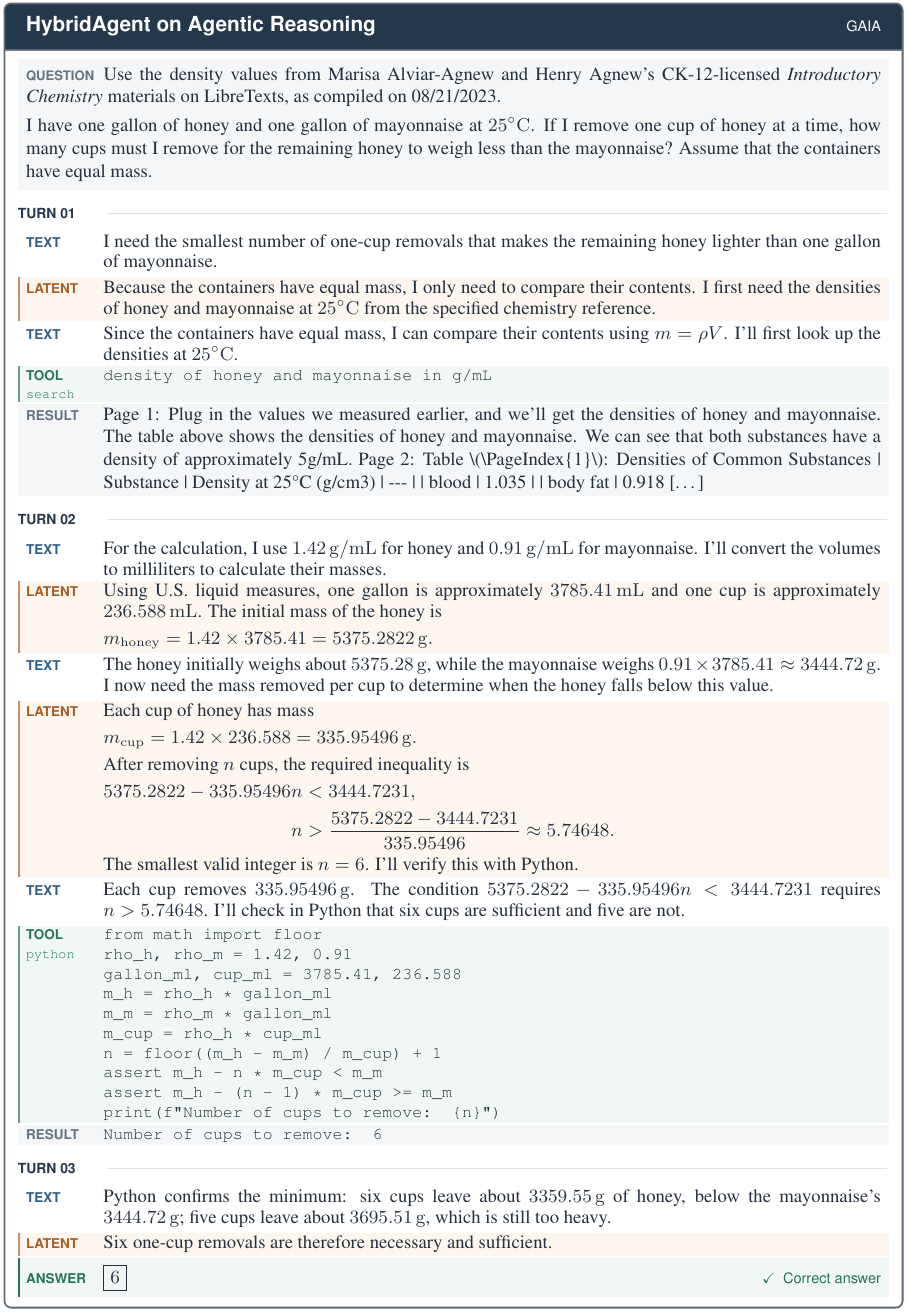}
    \label{fig:case3}
\end{figure}

\end{document}